\documentclass[final,5p,times,twocolumn]{elsarticle}

\usepackage{amssymb}
\usepackage{epstopdf}
\usepackage[english]{babel}
\biboptions{numbers,sort&compress}

 \usepackage{lineno}
\usepackage{subfigure}
\usepackage{graphicx}
\usepackage[figuresright]{rotating}
\usepackage{url}
\usepackage{pslatex}
\usepackage{bm}
\usepackage{amsmath,amssymb,amsthm}  
\usepackage{paralist}

\usepackage{algorithm}
\usepackage{algorithmic}        
\usepackage{multirow}

\newtheorem{lemma}{Lemma}

\newcommand{\e}{\begin{equation}}
\newcommand{\ee}{\end{equation}}
\newcommand{\en}{\begin{equation*}}
\newcommand{\een}{\end{equation*}}
\newcommand{\eqn}{\begin{eqnarray}}
\newcommand{\eeqn}{\end{eqnarray}}
\newcommand{\bmat}{\begin{bmatrix}}
	\newcommand{\emat}{\end{bmatrix}}
\newcommand{\BIT}{\begin{itemize}}
	\newcommand{\EIT}{\end{itemize}}
\newcommand{\eg}{{\it e.g.}}
\newcommand{\ie}{{\it i.e.}}

\newcounter{oursection}

\journal{Signal Processing}
\usepackage{color}
\begin{document}

\begin{frontmatter}



\title{An Efficient Method for Robust Projection Matrix Design}


\author[TH]{Tao Hong}
\ead{hongtao@cs.technion.ac.il}
\author[ZZ]{Zhihui Zhu}
\ead{zzhu@mines.edu}

	\address[TH]{Department
		of Computer Science, Technion - Israel Institute of Technology, Haifa, 32000, Israel.}
	\address[ZZ]{Department of Electrical Engineering,
		Colorado School of Mines, Golden, CO 80401 USA.} 


\begin{abstract}
Our objective is to efficiently design a robust projection matrix $\bm \Phi$ for the Compressive Sensing (CS) systems when applied to the signals that are not exactly sparse. The optimal projection matrix is obtained by mainly minimizing the average coherence of the equivalent dictionary. In order to drop the requirement of the sparse representation error (SRE) for a set of training data as in \cite{LLLBJH15} \cite{HBLZ16}, we introduce a novel penalty function independent of a particular SRE matrix. Without requiring of training data, we can efficiently design the robust projection matrix and apply it for most of CS systems, like a CS system for image processing with a conventional wavelet dictionary in which the SRE matrix is generally not available.  Simulation results demonstrate the efficiency and effectiveness of the proposed approach compared with the state-of-the-art methods. In addition, we experimentally demonstrate with natural images that under similar compression rate, a CS system with a learned  dictionary in high dimensions outperforms the one in low dimensions in terms of reconstruction accuracy. This together with the fact that our proposed method can efficiently work in high dimension suggests that a CS system can be potentially implemented beyond the small patches in sparsity-based image processing.
\end{abstract}


\end{frontmatter}

\section{Introduction}\label{S_1}
Since the beginning of this century, Compressive Sensing or Compressed Sensing (CS) has received a great deal of attention \cite{CRT06} - \cite{EK12}. Generally speaking, CS is a mathematical framework that addresses accurate recovery of a signal vector $\bm x\in \Re^{N}$ from a set of linear measurements
\e
\bm y = \bm \Phi\bm x \in \Re^{M}\label{e_1}
\ee
where $M\ll N$ and  $\bm \Phi\in \Re^{M\times N}$ is referred to as the projection or sensing matrix. CS has found many applications in the areas such as image processing, machine learning, pattern recognition,  signal detection/classification etc. We refer the reader to \cite{E10} \cite{EK12} and the references therein to find the related topics mentioned above.

Sparsity and coherence are two important concepts in CS theory.  We say a signal $\bm x$ of interest  approximately sparse (in some basis or dictionary)  if we can approximately express it as a linear combination of few columns (also called atoms) from a well-chosen dictionary:
\e
\bm x=\bm \Psi \bm \theta+\bm e\label{e_2}
\ee
where $\bm\Psi\in \Re^{N\times L}$ is the given or determined dictionary, $\bm \theta\in \Re^{L}$ is a sparse coefficient vector with few non-zero elements, and $\bm e\in \Re^{N}$ stands for the sparse representation error (SRE). In particular, the vector $\bm x$  is called (purely or exactly) $K$-sparse in $\bm \Psi$ if $\|\bm \theta\|_0 = K$ and $\bm e=\bm 0$ and approximately $K$-sparse in $\bm \Psi$ if $\|\bm \theta\|_0 = K$ and $\bm e$ has relatively small energy. Here, $\|\bm \theta\|_0$ denotes the number of non-zero elements in $\bm \theta$ and $\bm 0$ represents a vector whose entries are equivalent to $0$ . Through this paper, we say $\bm \theta$ is $K$-sparse if $\|\bm \theta\|_0 = K$ regardless whether $\bm e=\bm 0$.


Substituting the sparse model \eqref{e_2} of $\bm x$ into \eqref{e_1} gives
\e
\bm y=\bm\Phi\bm\Psi\bm \theta+\bm\Phi\bm e\triangleq \bm D\bm \theta+\bm\Phi\bm e\label{e_3}
\ee
where the matrix $\bm D=\bm\Phi\bm\Psi$ is referred to as the equivalent dictionary of the CS system and $\bm\epsilon\triangleq\bm \Phi\bm e$ denotes the projection noise caused by SRE. The goal of a CS system is to retrieve $\bm \theta$ (and hence $\bm x$) from the measurements $\bm y$. Due to the fact that $M\ll L$, solving $\bm y\approx\bm D\bm \theta$ for $\bm \theta$ is an undetermined problem which has an infinite number of solutions. By utilizing the priori knowledge that $\bm \theta$ is sparse, a CS system typically attempts to recover $\bm \theta$ by solving the following problem:
\e
{\bm \theta} = \arg\min\limits_{\tilde{\bm \theta}}\|\tilde{\bm \theta}\|_0, ~\text{s.t.}~\|\bm y-\bm D\tilde{\bm\theta}\|_2\leq \|\bm \epsilon\|_2\label{sparse_recov}
\ee
which can be solved by many efficient numerical algorithms including basis pursuit (BP), orthogonal matching pursuit (OMP), least absolute shrinkage and selection operator (LASSO) etc. All of the methods can be found in \cite{E10} \cite{ZE10} and the references therein.

To ensure exact recovery of $\bm \theta$ through \eqref{sparse_recov}, we need certain conditions on the equivalent dictionary $\bm D$. One of such conditions  is related to the concept of \emph{mutual coherence}. The mutual coherence of a matrix $\bm D\in \Re^{M\times L}$ is denoted by
\e
\mu(\bm D)\triangleq \max\limits_{1\leq i\neq j\leq L} |\bar{\bm G}(i,j)|\label{e_4}
\ee
where $\bar{\bm G}=\bar{\bm D}^\mathcal T\bar{\bm D}$ is called the Gram matrix of $\bar{\bm D}=\bm D\bm S_{sc}$ with $\bm S_{sc}$ a diagonal scaling matrix such that each column of $\bar{\bm D}$ is of unit length. Here $^\mathcal T$ represents the transpose operator. It is known that $\mu(\bm D)$ is lower bounded by the Welch bound $\underline{\mu}(\bm D)=\sqrt{\frac{L-M}{M(L-1)}}$, \ie, $\mu(\bm D)\in \left[\sqrt{\frac{L-M}{M(L-1)}},1\right]$. The mutual coherence $\mu(\bm D)$  measures the worst-case coherence between any two columns of $\bm D$ and is one of the fundamental quantities associated with the CS theory. As shown in \cite{E10}, when there is no projection noise (\ie, $\bm \epsilon = 0$), any $K$-sparse signal $\bm \theta$ can be exactly recovered  by solving the linear system \eqref{sparse_recov} as long as
\e
K<\frac{1}{2}\left[1+\frac{1}{\mu(\bm D)}\right]\label{e_5}
\ee
which indicates that a smaller $\mu(\bm D)$ ensures a CS system to recover the signal with a larger $K$. Thus, \cite{Barchiesi13} \cite{LiICASSP2017} proposed methods to design a dictionary with small mutual coherence. For a given dictionary $\bm\Psi$, the mutual coherence of the equivalent dictionary is actually determined or controlled by the projection matrix $\bm \Phi$. So it would be of great interest to design $\bm\Phi$ such that $\mu(\bm D)$ is minimized. Another similar indicator used to evaluate the average performance of a CS system is named \emph{average} mutual coherence $\mu_{av}$. The definition of $\mu_{av}$ is given as follows:
\en
\mu_{av}(\bm D)\triangleq \frac{\sum_{\forall (i,j)\in S_{av}}|\bar{\bm G}(i,j)|}{N_{av}}
\een
where $S_{av}\triangleq\{(i,j):{\bar \mu}\leq|\bar{\bm G}(i,j)|\}$ with $0\leq\bar{\mu}<1$ as a prescribed parameter and $N_{av}$ is the number of components in the index set $S_{av}$.

There has been much effort \cite{E07} - \cite{LZYCB13} devoted to designing an optimal $\bm\Phi$ that outperforms the widely used random matrix in terms of signal recovery accuracy (SRA).
However, all these methods are based on the assumption that the signal is exactly sparse under a given dictionary, which is not true for practical applications. It is experimentally observed that the sensing matrix designed by \cite{E07} - \cite{LZYCB13} based on mutual coherence results in inferior performance for real images (which are generally approximately but not exactly sparse under a well-chosen dictionary). To address this issue, the recent work in \cite{LLLBJH15} \cite{HBLZ16} proposed novel methods to design a robust projection matrix when the SRE exists.\footnote{We note that the approaches considered in \cite{LLLBJH15} \cite{HBLZ16} share the same framework. The difference is that in \cite{HBLZ16} the authors utilized an efficient iterative algorithm giving an approximate solution, while a closed form solution is derived in \cite{LLLBJH15}.} Through this paper, similar to what is used in \cite{LLLBJH15} \cite{HBLZ16}, a robust projection (or sensing) matrix means it is designed with consideration of possible SRE and hence the corresponding CS system yields superior performance when the SRE $\bm e$ in \eqref{e_2} is not nil. However, the approaches in \cite{LLLBJH15} \cite{HBLZ16} need the explicit value of the SRE on the training dataset, making them inefficient in several aspects. First, many practical CS systems with predefined analytical dictionaries (\eg, the wavelet dictionary, and the modulated discrete prolate spheroidal sequences (DPSS) dictionary for sampled multiband signals \cite{ZW}) actually do not involve any training dataset and hence no SRE available. In order to design the robust projection matrix for these CS systems using the framework presented in \cite{LLLBJH15} \cite{HBLZ16}, one has to first construct plenty of extra representative dataset for the explicit SRE with the given dictionary, which limits the range of applications. Second, even for the CS system with a dictionary learned typically on a large-scale dataset, we need a lot of memories and computations to store and compute with the huge dataset as well its corresponding SRE for designing a robust sensing matrix.  Moreover, if the CS system is applied to a dynamic dataset, \eg, video stream, it is practically impossible to store all the data and compute its corresponding SRE.  Therefore, the requirement of the explicit value of SRE for the training dataset makes the methods in \cite{LLLBJH15} \cite{HBLZ16} limited and inefficient for all the cases discussed above.

In this paper, to drop the requirement of the training dataset as well as its SRE, we propose a novel robust projection matrix framework only involving a predefined dictionary. With this new framework, we can efficiently design projection matrices for the CS systems mentioned above.
We stress that by efficient method for robust projection matrix design (which is the title of this paper), we are not providing an efficient method for solving the problems in \cite{LLLBJH15} \cite{HBLZ16}; instead we provide a new framework in which the training dataset and its corresponding SRE are not required any more. Experiments on synthetic data and real images demonstrate the proposed sensing matrix yields a comparable performance in terms of SRA compared with the ones obtained by \cite{LLLBJH15} \cite{HBLZ16}. 

Before proceeding, we first briefly introduce some notation used throughout the paper. MATLAB notations are adopted in this paper. In this connection, for a vector, $\bm v(k)$ denotes the $k$-th component of $\bm v$. For a matrix, $\bm Q(i,j)$ means the $(i,j)$-th element of matrix $\bm Q$, while $\bm Q(k,:)$ and $\bm Q(:,k)$ indicate the $k$-th row and column vector of $\bm Q$, respectively. We use $\bm I$ and $\bm I_L$ to denote an identity matrix with arbitrary and $L\times L$ dimension, respectively.  The $k$-th column of $\bm Q$ is also denoted by $\bm q_k$. $\text{trace}(\bm Q)$ denotes the calculation of the trace of $\bm Q$. { The Frobenius norm of a given matrix $\bm Q$ is $\|\bm Q\|_F=\sqrt{\sum_{i,j}\|\bm Q(i,j)\|^2}=\sqrt{\text{trace}(\bm Q^\mathcal T\bm Q)}$ where $\mathcal T$ represents the transpose operator.}
The definition of $l_p$ norm for a vector $\bm v\in \Re^{N}$ is $\|\bm v\|_p\triangleq \left(\sum\limits_{k=1}^N |\bm v(k)|^p\right)^{\frac{1}{p}},~~p\geq 1$.

The remainder is arranged as follows. Some preliminaries are given in Section \ref{S_2} to state the motivation of developing such a novel model. The proposed model which does not need the SRE is shown in Section \ref{S_3} and the corresponding optimal sensing problem is solved in this section. The synthetic and real data experiments are carried out in Section \ref{S_4} to demonstrate the efficiency and effectiveness of the proposed method. Some conclusions are given in Section \ref{S_5} to end this paper. 

\section{Preliminaries}\label{S_2}
A sparsifying dictionary $\bm \Psi$ for a given dataset $\{\bm x_k\}_{k=1}^P$ is usually obtained by considering the following problem
\e
\{\bm \Psi,\bm\theta_k\} =\arg \min\limits_{\tilde{\bm \Psi},\tilde{\bm \theta}_k}\sum\limits_{k=1}^{P} \|\bm x_k-\tilde{\bm \Psi}\tilde{\bm \theta}_k\|_2^2~~\text{s.t.}~~\|\tilde{\bm \theta}_k\|_0\leq K
\label{e_6}
\ee
which can be addressed by some practical algorithms \cite{TF11}, among which the popularly utilized are the K-singular value decomposition (K-SVD) algorithm \cite{AEB06} and the method of optimal direction (MOD) \cite{EAH99}.  As stated in the previous section, the SRE $\bm e_k=\bm x_k-\bm\Psi\bm\theta_k$ is generally not nil. We concatenate all the SRE $\{\bm e_k\}$ into an $N\times P$ matrix:
\en
\bm E\triangleq \bm X-\bm \Psi\bm \Theta
\een
which is referred to as the SRE matrix corresponding to the training dataset $\{\bm x_k\}$ and the learned dictionary $\bm \Psi$.

The recent work in \cite{LLLBJH15} \cite{HBLZ16} attempted to design a robust projection matrix with consideration of  the SRE matrix $\bm E$ and proposed to solve

\e
\bm \Phi=\arg\min\limits_{\tilde{\bm \Phi}}\|\bm I_L-\bm \Psi^\mathcal T\tilde{\bm \Phi}^\mathcal T\tilde{\bm \Phi}\bm\Psi\|_F^2+\lambda\|\tilde{\bm \Phi}\bm E\|_F^2 \label{e_7}
\ee
or
\e
\bm \Phi=\arg\min\limits_{\tilde{\bm \Phi}, \bm G\in H_\xi}\|\bm G-\bm \Psi^\mathcal T\tilde{\bm \Phi}^\mathcal T\tilde{\bm \Phi}\bm\Psi\|_F^2+\lambda\|\tilde{\bm \Phi}\bm E\|_F^2 \label{e_7_2}
\ee
where $H_\xi$ is the set of relaxed equiangular tight frames (ETFs):
\en
 H_\xi=\{\bm G|\bm G=\bm G^\mathcal T,~\bm G(i,i)=1, \forall\ i, \max\limits_{i\neq j}|\bm G(i,j)|\leq \xi  \}.
\een
Compared to \eqref{e_7} which requires the Gram matrix of the equivalent dictionary close to an identity matrix, \eqref{e_7_2} relaxes the requirement of coherence between the equivalent dictionary but is much harder to solve. See \cite{LLLBJH15} \cite{HBLZ16} for more discussions on this issue.

We remark that to ensure the designed sensing matrix by \eqref{e_7} or \eqref{e_7_2}  be robust to the SRE for all the signals of interest, the SRE matrix $\bm E$ should be well representative, \ie, we need sufficient number of training signals $\bm x_k$.   As stated in  \cite{LLLBJH15} \cite{HBLZ16}, these methods (\eqref{e_7} and \eqref{e_7_2}) can be applied naturally when the dictionary is learned by algorithms like K-SVD with plenty of training data $\{\bm x_k\}$, since the SRE $\bm E$ is available without any additional effort. However, this could be prohibitive when the CS system with an analytic dictionary is applied to some arbitrary signals (but still they are approximately sparse in this dictionary), since there are no sufficient number of data available to obtain the SRE matrix $\bm E$. For example, one may only want to apply the CS system to an arbitrary image with the wavelet dictionary. Also, these methods are prohibitive for a dictionary trained on large datasets with millions of training samples and in a dynamic CS system for streaming signals. To train such a dictionary, we have to conduct online algorithms \cite{B98} - \cite{MBPS10} which typically apply stochastic gradient method where in each iteration a randomly selected tiny part of the training signals called mini-batch instead of the whole data is utilized for computing the expected gradient. In these cases, additional efforts are needed to obtain the SRE matrix $\bm E$ and it is usually prohibitive to compute and store $\bm E$ for all the training dataset. All of these situations make the approach proposed in \cite{LLLBJH15} \cite{HBLZ16} become limited.

Aiming to obtain a neural network well expressing the signals of interest, an empirical strategy widely used by deep learning community is to utilize a huge training dataset so that the network can extract more important features and avoid over-fitting. Similar to this phenomenon, a dictionary trained on a huge dataset is also expected to contain more features of the represented signal. Simulation results shown in Section \ref{S_4} demonstrate that a CS system with such a dictionary and designing a projection matrix designed on this dictionary yields a higher reconstruction accuracy on natural images than the one with a dictionary obtained from a small dataset. The recent work in \cite{SOZE16} \cite{SE16} stated that a dictionary learned with larger patches (\eg, $64\times64$) on a huge dataset can better capture features in natural images.\footnote{The dimension of a dictionary in such a case becomes high compared with the moderate dictionary size shown in \cite{AEB06}. In fact, the name of a high dimensional dictionary in this paper means the dictionary obtained by training  on a larger size of represented signal.} In this paper, we also attempt to experimentally investigate the performance of designing a projection matrix on a high dimensional dictionary.  These also motivate us to develop an efficient method for designing a robust sensing matrix without the requirement of the SRE matrix $\bm E$ as it is not easy to obtain for the above two situations. 

In the next section, we provide a novel framework to efficiently design a robust sensing matrix, and more importantly, it can be applied to the situation when the SRE matrix $\bm E$ is not available.

\section{A Novel Approach to Projection Matrix Design}\label{S_3}
In this section, we provide an efficient robust sensing matrix design approach which drops the requirement of training signals and their corresponding SRE matrix $\bm E$. Our proposed framework is actually inspired by \eqref{e_7} and \eqref{e_7_2} from the following two aspects.

\subsection{A Novel Framework for Robust Projection Matrix Design}
First note that the energy of SRE $\|\bm E\|_F^2$ is usually very small since the learned sparsifying dictionary is assumed to sparsely represent a signal well as in \eqref{e_2}. Otherwise if $\|\bm E\|_F^2$ is very large, it indicates that the dictionary is not well designed for this class of signals and it is possible that this class of signal can not be recovered from the compressive measurements no matter what projection matrix is utilized. It follows from the norm consistent property that
\e
\|\bm \Phi \bm E\|_F \leq \|\bm \Phi\|_F\|\bm E\|_F
\label{eq:PhiE to Phi 1}
\ee
which implies informally that a smaller sensing matrix $\|\bm \Phi\|_F$ yields a smaller projected SRE $\|\bm \Phi \bm E\|_F$.

Also as illustrated before that the amount of training data should be sufficient so that they can represent the class of targeted signals, the energy in the corresponding SRE matrix $\bm E$ should spread out in every elements. In other words, one can view the expected SRE as an additive Guassian white noise. In this situation, we have the following result.

\begin{lemma}
	Suppose $\bm E(:,k) = \bm e_k, \forall k=1,\cdots P$ are i.i.d Gaussian random vectors with each of zero mean and covariance $\sigma^2\bm I$. Then for any $\bm \Phi\in R^{M\times N}$, we have
	\e
	\mathbb E\left[\|\bm \Phi \bm E\|_F^2\right] = P\sigma^2 \| \bm \Phi \|_F^2.
	\label{eq:PhiE to Phi 2}
	\ee
	where $\mathbb E$ denotes the expectation operator.

Moreover, when the number of training samples $P$ approaches to $\infty$, we have $\frac{\|\bm \Phi \bm E\|_F^2}{P}$ converges in probability and almost surely to $\sigma^2 \|\bm \Phi\|_F^2$. In particular,
\e
\sqrt{p}\left( \frac{\|\bm \Phi \bm E\|_F^2}{P} - \sigma^2 \|\bm \Phi\|_F^2\right) \xrightarrow{d} {\cal N}(0,2\sigma^2\|\bm \Phi \bm \Phi^{\cal T}\|_F^2)
\label{eq:lem 2}\ee
where ${\cal N}(\mu,\varsigma)$ denotes the Gaussian distribution of mean $\mu$ and variance $\varsigma$, and $\xrightarrow{d}$ means convergence in distribution.

\label{Lemma:omit_E}
\end{lemma}

\begin{proof} 
For each $k$, we first define $\bm d_k = \bm \Phi \bm e_k$. Since $\bm e_k\sim {\mathcal N}(0,\sigma^2\bm I)$, we have $\bm d_k  \sim {\mathcal N}(0,\sigma^2\bm\Phi \bm \Phi^{\cal T})$. Let $\bm \Phi \bm \Phi^{\cal T} = \bm Q \bm \Lambda \bm Q^{\cal T}$ be an eigendecomposition of $\bm \Phi \bm \Phi^{\cal T}$, where $\bm \Lambda$ is an $M\times M$ diagonal matrix with the non-negative eigenvalues $\lambda_1,\ldots,\lambda_M$ along its diagonal. We have
\[
\|\bm d_k\|_2^2 = \|\bm Q^{\cal T}\bm d_k\|_2^2
\]
and
\[
\bm Q^{\cal T}\bm d_k \sim {\mathcal N}(0,\sigma^2\bm \Lambda).
\]
For convenience, we define new random variables $\bm c = \bm Q^{\cal T}\bm d_k $ and $b_1 = \frac{1}{\lambda_1\sigma^2}\bm c^2(1), b_2 = \frac{1}{\lambda_1\sigma^2}\bm c^2(2), \ldots, b_M = \frac{1}{\lambda_1\sigma^2}\bm c^2(M)$. It is clear that $b_1$, $b_2$, \ldots, $b_M$ are independent random variables of $\chi_1^2$ distribution, the chi-squared distribution with $1$ degree of freedom.

Now we compute the mean of $\|\bm \Phi \bm e_k\|_F^2$:
\e\begin{split}
\mathbb E[\|\bm \Phi \bm e_k\|_F^2] &= \mathbb E[\| \bm d_k\|_F^2] = \mathbb E[\| \bm c\|_F^2]\\
& = \sum_{i=1}^M \lambda_i \sigma^2\mathbb E[b_i]= \sum_{i=1}^M \lambda_i \sigma^2\\
&= \text{trace}(\sigma^2\bm \Phi \bm \Phi^{\cal T}) = \sigma^2 \|\bm \Phi\|_F^2
\end{split}\label{eq:proof 1}\ee
where the second line we utilize $\mathbb E[\chi_1^2] = 1$.
The variance of $\|\bm \Phi \bm e_k\|_F^2$ is given by:
\e\begin{split}
\text{Var}[\|\bm \Phi \bm e_k\|_F^2] &= \text{Var}[\| \bm d_k\|_F^2] = \text{Var}[\| \bm c\|_F^2]\\
& = \sum_{i=1}^M \lambda_i^2 \sigma^4\text{Var}[b_i] = 2\sum_{i=1}^M \lambda_i^2 \sigma^4\\
& = 2\text{trace}(\sigma^4\bm \Phi \bm \Phi^{\cal T}\bm \Phi \bm \Phi^{\cal T})\\
& = 2\sigma^4 \|\bm \Phi  \bm \Phi^{\cal T}\|_F^2
\end{split}\label{eq:proof 2}\ee
where the second line we utilize $\text{Var}[\chi_1^2] = 2$, and the third line follows because
\[
\bm \Phi \bm \Phi^{\cal T}\bm \Phi \bm \Phi^{\cal T} = \bm Q \bm \Lambda^2 \bm Q^{\cal T}.
\]

Thus, we obtain \eqref{eq:PhiE to Phi 2} by noting that
	\[
	\mathbb E[\|\bm \Phi \bm E\|_F^2] = \mathbb E\left[\sum_{k=1}^P \|\bm \Phi \bm e_k\|_2^2 \right] = P\sigma^2 \|\bm \Phi\|_F^2
	\]
It follows from \eqref{eq:proof 1} and \eqref{eq:proof 2} that $\left\{\|\bm \Phi \bm e_1\|_2^2,\ldots, \|\bm \Phi \bm e_P\|_2^2\right\}$ is a sequence of independent and identically distributed random variable drawn from distributions of expected values given by $\sigma^2 \|\bm \Phi\|_F^2$ and variances given by $2\sigma^4 \|\bm \Phi  \bm \Phi^{\cal T}\|_F^2$. Thus, by the law of large numbers~\cite{CR2002}, the average $\frac{\|\bm \Phi \bm E\|_F^2}{P}$ converges in probability and almost surely to the expected value $\sigma^2 \|\bm \Phi\|_F^2$ as $P\rightarrow \infty$. Finally, the central limit theorem~\cite{CR2002} establishes that as $P$ approaches infinity, the random variables $\sqrt{P}(\frac{\|\bm \Phi \bm E\|_F^2}{P} - \sigma^2 \|\bm \Phi\|_F^2)$ converges in distribution to a normal ${\cal N}(0,2\sigma^2\|\bm \Phi \bm \Phi^{\cal T}\|_F^2)$.

\end{proof}
In words, Lemma~\ref{Lemma:omit_E} indicates that when the number of training samples approaches to infinity, $\|\bm \Phi \bm E\|_F^2$ is proportional to $\|\bm \Phi\|_F^2$. Inspired by \eqref{eq:PhiE to Phi 1}-\eqref{eq:lem 2}, it is expected that without any training signals and their corresponding SRE matrix $\bm E$, a robust projection matrix can be obtained by solving the following problem
\e
\bm \Phi=\arg\min\limits_{\tilde{\bm \Phi}}f(\tilde{\bm\Phi})\equiv\|\bm I_L-\bm \Psi^\mathcal T\tilde{\bm \Phi}^\mathcal T\tilde{\bm \Phi}\bm\Psi\|_F^2+\lambda\|\tilde{\bm \Phi}\|_F^2 \label{e_9}
\ee
or
\e
\bm \Phi=\arg\min\limits_{\tilde{\bm \Phi}, \bm G\in H_\xi} f(\tilde{\bm\Phi},\bm G)\equiv \|\bm G-\bm \Psi^\mathcal T\tilde{\bm \Phi}^\mathcal T\tilde{\bm \Phi}\bm\Psi\|_F^2+\lambda\|\tilde{\bm \Phi}\|_F^2 \label{e_9_2}
\ee
Here, with abuse of notation, we use both $f$ to denote the objective function in \eqref{e_9} and \eqref{e_9_2}. However, it should be clear from the context as we always use $f(\tilde{\bm \Phi})$ to represent the one in \eqref{e_9} and $f(\tilde{\bm \Phi}, \bm G)$ to represent the one in \eqref{e_9_2}. Since the more training samples can better represent the signals of interest and the SRE, \eqref{eq:lem 2} indicates that the sensing matrices obtained by \eqref{e_9} and \eqref{e_9_2} are more robust to SRE than the ones obtained by \eqref{e_7} and \eqref{e_7_2}. This is demonstrated by experiments in Section~\ref{S_4}. The numerical algorithms is presented to solve \eqref{e_9} and \eqref{e_9_2} in the following section.
\subsection{Efficient Algorithms for Solving \eqref{e_9} and \eqref{e_9_2}}
Note that $f(\tilde{\bm\Phi})$ is a special case of $f(\tilde{\bm\Phi},\bm G)$ with $\bm G = \bm I_L$. Thus, we first consider solving
\e
\min_{\tilde{\bm\Phi}}~f(\tilde{\bm\Phi},\bm G) = \|\bm G-\bm \Psi^\mathcal T\tilde{\bm \Phi}^\mathcal T\tilde{\bm \Phi}\bm\Psi\|_F^2+\lambda\|\tilde{\bm \Phi}\|_F^2
\label{eq:min f} \ee
 with an arbitrary $\bm G$. {To that end, we introduce a low-rank minimization problem
\e\begin{split}
\min_{\bm A} g(\bm A, \bm G)&\equiv \|\bm G-\bm \Psi^\mathcal T\bm A \bm\Psi\|_F^2+\lambda \text{trace}(\bm A)\\
& \text{s.t.}\ \text{rank}(\bm A)\leq M,  \bm A\succeq 0
\end{split}\label{eq:min g}\ee
By eigendecomposition of $\bm A$, it is clear that \eqref{eq:min f} is equivalent to \eqref{eq:min g}. The problem \eqref{eq:min f} is often referred to as the factor problem of \eqref{eq:min g}. Also note that $g(\bm A,\bm G)$ is a convex function of $\bm A$ for any fixed $\bm G$, though the problem \eqref{eq:min g} is nonconvex because of the rank constraint. The recent work~\cite{Zhu2017} has shown that a number of iterative algorithms (including gradient descent) can provably solve the factored problem (\ie, \eqref{eq:min f}) for a set of low-rank matrix optimizations (\ie, \eqref{eq:min g}).} Thus, in this paper, the \emph{Conjugate-Gradient} (CG) \cite{NW06} method is utilized to solve \eqref{eq:min f}.\footnote{We note that both of the methods shown in \cite{LLLBJH15} \cite{HBLZ16} for solving \eqref{eq:min f} need to calculate the inversion of $\bm \Psi\bm\Psi^\mathcal T$. However, in practice, the learned dictionary sometimes is  ill-conditioned, which may cause numerical instable problem if directly applying their methods. Thus, as global convergence of many local search algorithms for solving similar low-rank optimizations is guaranteed in~\cite{Zhu2017},  CG is chosen to solve \eqref{eq:min f}. Obviously, if the aforementioned problem does not happen in practical cases, the method in \cite{LLLBJH15} \cite{HBLZ16} can be used to address \eqref{e_9}. Moreover, we will show that CG and the methods shown in \cite{LLLBJH15} \cite{HBLZ16} yield a similar solution in the following experiments.} 
The gradient of $f(\tilde{\bm\Phi},\bm G)$ in terms of $\tilde{\bm \Phi}$ is given as follows:
\e
\nabla_{\tilde{\bm\Phi}} f(\tilde{\bm \Phi},\bm G)=2\lambda\bm\tilde{\bm \Phi}-4\tilde{\bm \Phi}\bm\Psi\bm G\bm\Psi^\mathcal T+4\tilde{\bm \Phi}\bm\Psi\bm\Psi^\mathcal T\tilde{\bm \Phi}^\mathcal T \tilde{\bm \Phi}\bm\Psi\bm\Psi^\mathcal T\label{e_10}
\ee
After obtaining the gradient of $f(\tilde{\bm\Phi},\bm G)$, the toolbox \emph{minFunc}\footnote{We note that minFunc is a stable toolbox that can be efficiently applied with millions of variables.} \cite{MS05} is utilized to solve \eqref{eq:min f} with CG method. We note that the gradient-based method only involves simple matrix multiplication in \eqref{e_10}, without requiring performing SVD and matrix inversion.  Hence it is also suitable for designing the projection matrix for a CS system working on high dimensional signals.

We now turn to solve \eqref{e_9_2} which has two variables $\tilde{\bm \Phi}$ and $\bm G\in H_\xi$. A widely used strategy for such problems is the alternating minimization \cite{E07} \cite{AFS12} \cite{LLLBJH15} \cite{HBLZ16}. The main idea behind alternating minimization for \eqref{e_9_2} is that we keep one variable constant (say $\tilde{\bm \Phi}$), and optimize over the other variable (say ${\bm G}$). Once $\bm G$ is fixed, as we explained before, we utilize CG method to solve \eqref{eq:min f}. On the other hand, the solution to $\min_{\bm G} f(\tilde{\bm \Phi}, \bm G)$ can be simply obtained by projecting the Gram matrix of the equivalent dictionary onto the set $H_\xi$ when we fix $\tilde {\bm \Phi}$. The main steps of the algorithm are outlined in Algorithm 1.

\begin{description}
	\item[Algorithm $1$]
	\item[Initialization:] Set $k=1$, $\bm\Phi_0$ as a random one and the number of iterations $Iter$.
	\item[Step I:] Set $\tilde{\bm G}_k = \bm \Psi^\mathcal T\bm \Phi_{k-1}^\mathcal T\bm \Phi_{k-1}\bm \Psi$ and then project it onto the set $H_\xi$:
	\en
	\bm G_k(i,j)=\left\{
	\begin{array}{ll}
      1, & i=j,\\
		\tilde{\bm G}_k(i,j),& 	i\neq j, |\tilde{\bm G}_k(i,j)|\leq\xi,\\
		\xi\cdot\text{sign}(\tilde{\bm G}_k(i,j)),& i\neq j, |\tilde{\bm G}_k(i,j)|>\xi
	\end{array}
	\right.
	\een
	 where $\text{sign}(\cdot)$ is a sign function.
	\item[Step II:] Solve $\bm\Phi_k=\arg\min_{\tilde{\bm \Phi}}f(\tilde{\bm \Phi},\bm G_k)$ with CG.

If $k<Iter$, set $k=k+1$ and go to Step I. Otherwise, terminate the algorithm and output $\Phi_{Iter}$.
\end{description}


{\bf Remarks:}
\begin{itemize}
	\item It is clear that this approach is independent of training data and can be utilized for most of CS systems as long as the sparsifying dictionary $\bm \Psi$ is given.
	\item Even in the case where the SRE matrix $\bm E$ is available, it is much easier and more efficient to solve \eqref{e_9} than \eqref{e_7} since typically the number of columns in $\bm E$ is dramatically greater than the size of $\bm \Psi$ and $\bm \Phi$, \ie, $P\gg M,N,L$.
	\item Simulation results with synthetic data and natural images (where the SRE matrix $\bm E$ is available) show that the proposed method also yields a comparable performance to or outperforms the methods in \cite{LLLBJH15} \cite{HBLZ16} in terms of SRA. Moreover, the experiments on natural images show that designing a projection matrix on a given dictionary which is learned with large dataset or high-dimensional training signals can improve SRA significantly with the same compression rate $\frac{M}{N}$. However, it requires a great deal of memories to store the SRE matrix $\bm E$ for either large dateset or high-dimensional training data.
\end{itemize}

\section{Simulation Results}\label{S_4} 
In this section, we perform a set of experiments on synthetic data and natural images to demonstrate the performance of the CS system with projection matrix designed by the proposed methods. For convenience, the corresponding CS systems are denoted by $CS_{MT}$ with $\tilde {\bm \Phi}$ obtained via \eqref{e_9} and $CS_{MT-ETF}$ with $\tilde {\bm \Phi}$ obtained via \eqref{e_9_2}, and are compared with the following CS systems: $CS_{randn}$ with a random projection matrix, $CS_{LH}$ with the sensing matrix obtained via \eqref{e_7} \cite{LLLBJH15}, $CS_{LH-ETF}$ with the sensing matrix obtained via \eqref{e_7_2} \cite{LLLBJH15},  and  $CS_{DCS}$ \cite{DCS09}. It was first proposed in \cite{DCS09} that \emph{simultaneously} optimizing  $\bm \Phi$ and $\bm \Psi$ for a CS system results in better performance in terms of SRA. In the sequel, we also examine this strategy in natural images and the corresponding CS system is denoted by $CS_{S-DCS}$.\footnote{In our experiment, the coupling factor utilized in $CS_{S-DCS}$ is set to $0.5$ which is the best value in our setting.
 Generally speaking, $CS_{S-DCS}$ should have a best performance in terms of SRA because it optimizes projection matrix and dictionary simultaneously. Thus, the performance of $CS_{S-DCS}$ serves as the indicator of the best performance can be achieved by other CS systems that only consider optimizing the projection matrix. }
For simplicity, the parameter $\xi$ in $H_\xi$ is set to Welch bound in the following experiments.

The SRA is evaluated in terms of the peak signal-to-noise ratio (PSNR) \cite{E10}
\en
\varrho_{psnr}\triangleq 10\times\log 10\left[\frac{(2^r-1)^2}{\varrho_{mse}}\right]\text{dB}
\een
with $r=8$ bits per pixel.
We also utilized the measure $\varrho_{mse}$:
\en
\varrho_{mse}\triangleq \frac{1}{N\times P}\sum\limits_{k=1}^P\| \tilde{\bm x}_k- \bm x_k\|_2^2
\een
where $\bm x_k$ is the original signal, $\tilde{\bm x}_k=\bm\Psi\tilde{\bm\theta}_k$ stands for the reconstructed signal with $\tilde{\bm \theta}_k$ the solution of \eqref{sparse_recov}, and $P$ is the number of patches in an image or the testing data.

\subsection*{A. Synthetic Data Experiments}
An $N\times L$ dictionary $\bm\Psi$ is generated with normally distributed entries and then is normalized so that each column has unit $l_2$ norm. We also generate a random $M\times N$ matrix $\bm\Phi_0$ (where each entry has Gaussian distribution of zero-mean and variance $1$)
 as the initial condition  for all of the aforementioned projection matrices. $\bm \Phi_0$ is also utilized as the sensing matrix in $CS_{randn}$.

The synthetic data for training and testing is obtained as follows. A set of $2P$ $K$-sparse vectors $\{\bm\theta_k\in\Re^L\}$ is generated as the sparse coefficients where each non-zero elements of $\{\bm\theta_k\}$ is randomly positioned with a Gaussian distribution of zero-mean and unit variance. The set of signal vectors $\{\bm x_k\}$ is produced with $\bm x_k=\bm\Psi\bm s_k+\bm e_k\triangleq \bm x_k^{(0)}+\bm e_k,~\forall k$, where $\bm\Psi$ is the given dictionary and $\bm e_k$ is the random noise with Gaussian distribution of zero-mean and variance $\sigma_e^2$ to yield different signal-to-noise ration (SNR) (in dB) of the signals. Clearly, $\bm x_k^{(0)}$ is exactly $K$-sparse in $\bm \Psi$, while $\bm x_k$ is approximately $K$-sparse in $\bm \Psi$.

Denote $\bm X=\bm X^{(0)}+\bm \Delta$ as the signal matrix of dimension $N\times 2P$, where $\bm X^{(0)}(:,k) =\bm x_k^{(0)}$ and $\bm \Delta(:,k) = \bm e_k$. We use the SRE matrix $\bm E=\bm\Delta(:,1:P)$ in \eqref{e_7} and \eqref{e_7_2} whose solutions are used for $CS_{LH}$ and $CS_{LH-ETF}$, respectively. The data $\bm X(:,P+1:2P)$ is utilized for testing the CS systems. The measurements $\{\bm y_k\}$ are obtained by $\bm y_k=\bm\Phi\bm X(:,P+k), ~\forall k\in(0,P]$ where $\bm\Phi$ is the projection matrix of the CS systems.  For simplicity, OMP is chosen to solve the sparse coding problem throughout the experiments.

With the synthetic data, we conduct three set of experiments to demonstrate the performance of our proposed framework for robust projection matrix design, \ie, \eqref{e_9} and \eqref{e_9_2}. In these three set of experiments, we respectively show the convergence of CG method, the effect of $\lambda$ and the signal recovery accuracy of the proposed projection matrices $CS_{MT}$ and $CS_{MT-ETF}$ versus different SNR of the signals.

$1$) \emph{Convergence Analysis}:  Let $M = 20$, $N=60$, $L=100$ and $K=4$.We utilize CG to solve \eqref{e_9}. We note that a random dictionary with well-conditioned is chosen and thus we also compute the closed-form solution shown in \cite{LLLBJH15} for \eqref{e_9}. The objective value obtained by the closed-form solution is denoted by $f^*$ and is compared with the CG method. The evolution of $f(\bm\Phi)$ for different $\lambda$ is shown in Figure \ref{sythetic_convergence}. { We note that different $\lambda$ results in different functions $f(\bm\Phi)$ and hence different $f^*$.} We observe global convergence of CG method for solving \eqref{e_9} { with all the choices of $\lambda$}.

\begin{figure}[htb!]
	\centering
	\includegraphics[width=0.5\textwidth]{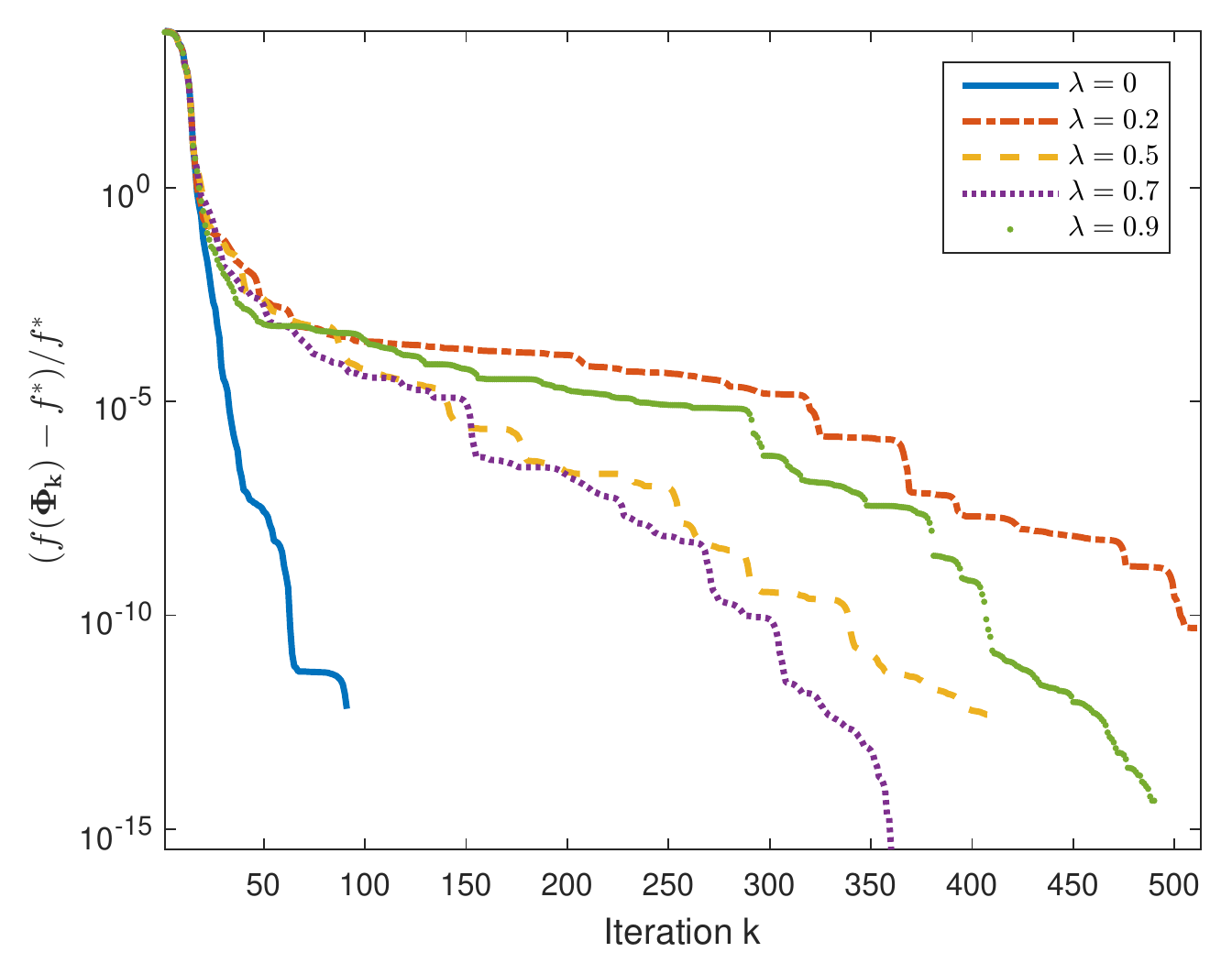}\\
	\caption{Evolution of $f(\bm\Phi)$ in \eqref{e_9} for different values of $\lambda$ versus iteration, where the sparsity level is set to $K=4$. Here $k$ represents the iteration of the CG method.}\label{sythetic_convergence}
\end{figure}

$2$) The Choice of $\lambda$:  With $M = 20$, $N=60$, $L=80$, $K=4$ and SNR = $15$ dB, we check the effect of the trade-off parameter $\lambda$ in terms of $\varrho_{mse}$  for $CS_{MT}$ and $CS_{MT-ETF}$. The $\lambda$ is chosen from $0$ to $2$ with step size $0.01$. The evaluation of $\varrho_{mse}$ versus different $\lambda$ is depicted in Figure \ref{sythetic_choice_lambda}.

\vspace{3pt}
\noindent{\emph {Remark $1$:}}
\begin{itemize}
	\item As seen from Figure \ref{sythetic_choice_lambda}, different $\lambda$ yields different performance in terms of $\varrho_{mse}$ for this practical situation where the SNR is $15$dB. It is clear that a proper choice of $\lambda$ results in  significantly better performance than other values, especially for $CS_{MT-ETF}$. Clearly, the advantage of the proposed method is shown by comparing the cases for $\lambda=0$ and other values of $\lambda$ as the former corresponds to the traditional approaches which do not take the SRE into account. In the sequel, we simplicity search the best $\lambda$ (with which the CS systems attain the minimal $\varrho_{mse}$ for the test data) within $(0,1]$ for each experiment setting.
	\item According to this experiment, if $\lambda$ is well chosen, $CS_{MT-ETF}$ has better performance than $CS_{MT}$ in terms of $\varrho_{mse}$. However,  the performance of $CS_{MT-ETF}$ is more sensitive than with $\lambda$  than $CS_{MT}$. We will show in the next experiment that the performance of $CS_{MT-ETF}$  outperforms $CS_{MT}$  in synthetic data when the SNR is not too small. However, for natural images which have relatively large SRE, $CS_{MT}$ always has better performance than $CS_{MT-ETF}$. This phenomenon is also observed for $CS_{LH}$ and $CS_{LH-ETF}$ in \cite{LLLBJH15} \cite{HBLZ16}. Thus, we only consider the performance of $CS_{MT}$ and $CS_{LH}$ for the natural images in next section.

\end{itemize}

\begin{figure}[htb!]
	\centering
	\includegraphics[width=0.5\textwidth]{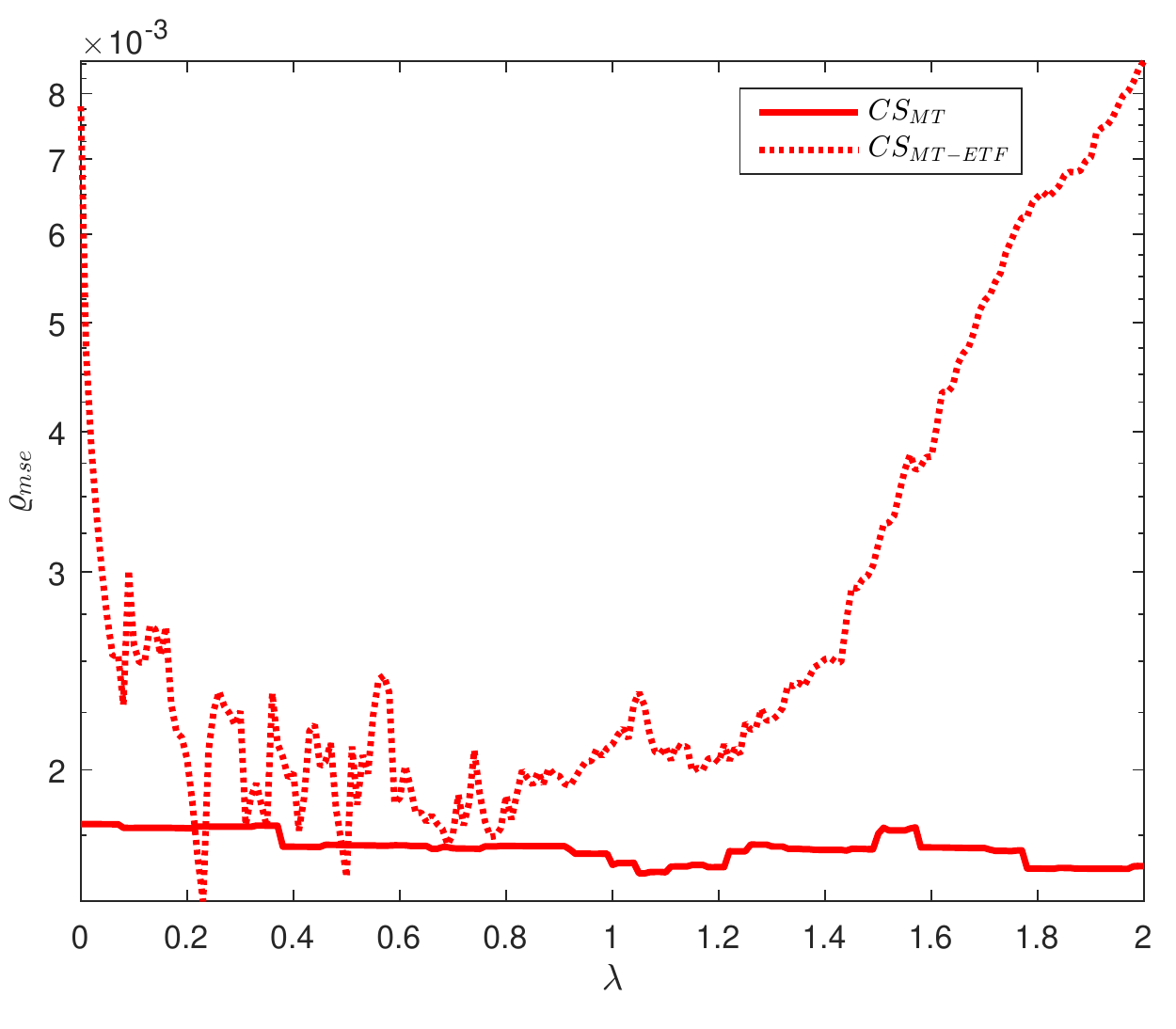}
	\caption{Performance evaluation: $\varrho_{mse}$ versus the values of $\lambda$, where the sparsity level is set to $K=4$ and SNR = $15$ dB.}\label{sythetic_choice_lambda}
\end{figure}

$3$) Signal Recovery Accuracy Evaluation: With $M = 20$, $N=60$, $L=80$, $K=4$ and $P=1000$, we compare our CS systems $CS_{MT}$ and $CS_{MT-ETF}$ with other CS systems for SNR varying from $5$ to $45$ dB. Figure \ref{sythetic_evaluation_mse} displays signal reconstruction error $\varrho_{mse}$ versus SNR for all six CS systems.

\begin{figure}[htb!]
	\centering
	\includegraphics[width=0.5\textwidth]{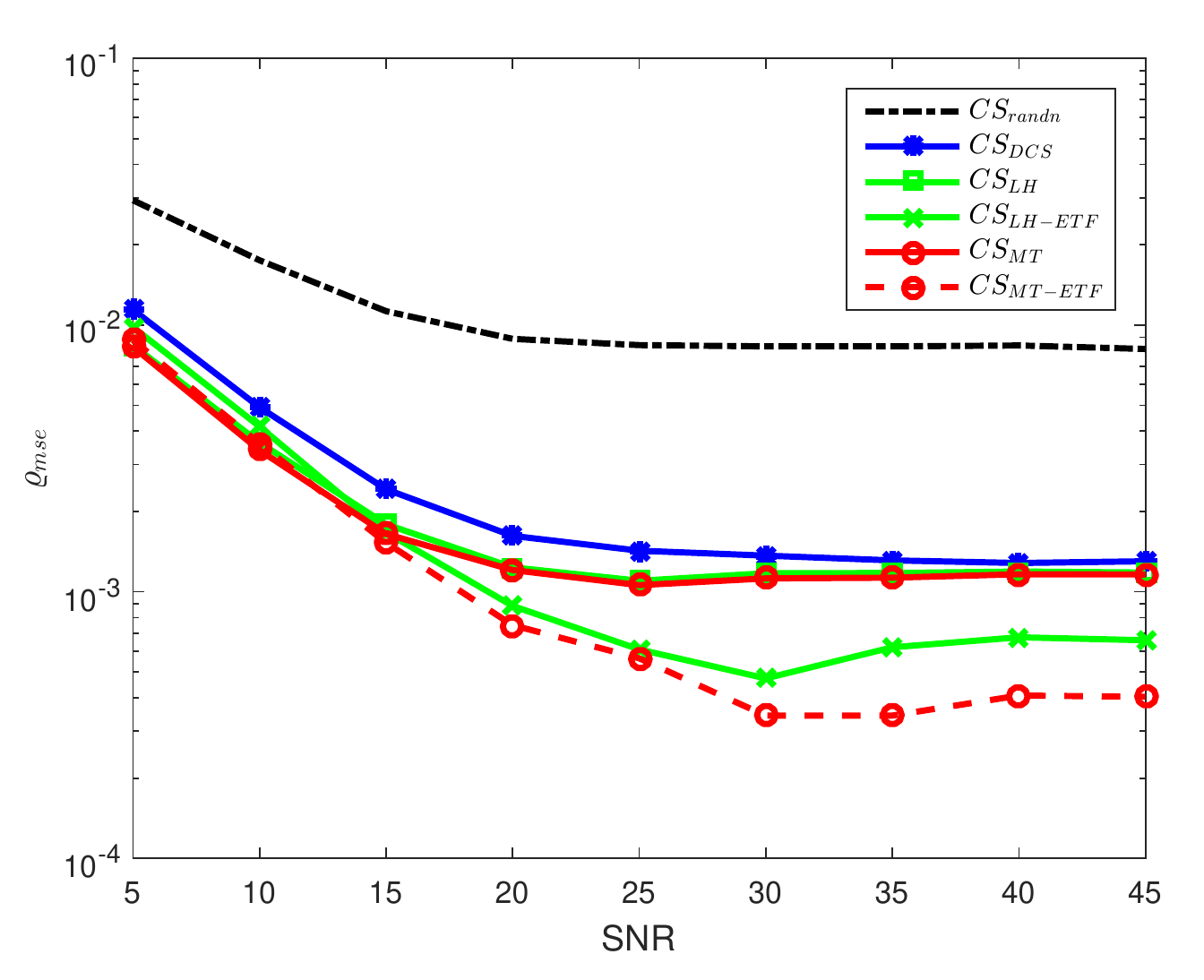}
	\caption{Reconstruction error $\varrho_{mse}$ versus SNR for each of the six CS systems.}\label{sythetic_evaluation_mse}
\end{figure}

\vspace{3pt}
\noindent{\emph {Remark $2$:}}
\begin{itemize}
	\item It is clear that the sensing matrices obtained via \eqref{e_9} and \eqref{e_9_2} have at least similar performance to the ones obtained via \eqref{e_7} and \eqref{e_7_2} \cite{LLLBJH15,HBLZ16}, though our proposed framework does not utilize the SRE matrix $\bm E$. We also observe that $CS_{MT-ETF}$ outperforms $CS_{LH-ETF}$ when SNR is larger than $15$ dB. This demonstrates the effectiveness of our proposed framework in Section \eqref{S_3} and verifies our argument in {\bf Lemma  \ref{Lemma:omit_E}} that the sparse representation error is not explicitly required.
	\item As seen from Figure \ref{sythetic_evaluation_mse}, $CS_{MT}$ has slightly better performance than $CS_{MT-ETF}$ when the SNR is smaller than $15$ dB. In other words, we recommend to utilize $CS_{MT}$ (with the sensing matrix obtained via \eqref{e_9}) when the sparse representation error is relatively large, \eg, natural images, which meets our claims in \emph{Remark $1$}.

\end{itemize}

\subsection*{B. Natural Images Experiments}
In this section, three set of experiments are conducted on natural images. Through these experiments, we verify the effectiveness of the proposed framework for robust sensing matrix design in Section \ref{S_3} and demonstrate the reason for dropping the requirement on the SRE matrix $\bm E$. As we explained before, since the SRE is relatively large for natural images, $CS_{MT}$ and $CS_{LH}$ are respectively superior to $CS_{MT-ETF}$ and $CS_{LH-ETF}$. Thus, we only show the results for $CS_{MT}$ and $CS_{LH}$.

In the first set of experiments, we compare the performance of $CS_{MT}$ and $CS_{LH}$ when a set of training signals and the corresponding SRE matrix $\bm E$ are available. In the second set of experiments, we design the projection matrix with a dictionary learned on a much larger training dataset. The performance of CS systems with a higher dimensional dictionary is given in the \emph{Experiment C}. We observe that a CS system with a higher dimensional dictionary and a projection matrix designed by our proposed algorithm yields better SRA under the same compression rate. Both training and testing datasets used in these three set of experiments are extracted as follows from the LabelMe database \cite{RTF}. Note that \emph{Data I} is extracted with small patches and \emph{Data II} is obtained by sample larger patches for the third experiment. 

\vspace{3pt}
\noindent\emph{Training Data I:} A set of $8\times 8$ non-overlapping patches is obtained by randomly extracting $400$ patches from each image in the whole LabelMe training dataset. We arrange each patch of $8\times 8$ as a vector of $64\times 1$. A set of $400\times2920=1.168\times10^6$ training samples is obtained to train the sparsifying dictionary.

\vspace{3pt}
\noindent\emph{Testing Data I:} A set of $8\times 8$ non-overlapping patches is obtained by randomly extracting $15$ patches from $400$ images in LableMe testing dataset as the testing data.
\vspace{3pt}

\vspace{3pt}
\noindent\emph{Training Data II:} The training data contains a set of $16\times 16$ non-overlapping patches which are obtained by randomly extracting  $400$ patches from the whole images in the LabelMe training dataset.  Each $16\times 16$ patch is then arranged as a length-$256$ vector. A set of $1.168\times10^6$ training samples is utilized.

\vspace{3pt}
\noindent\emph{Testing Data II:} The testing data is extracted in the same way for the training data but from the LabelMe testing dataset. We randomly extract $8000$ testing samples from $400$ images with each sample an $16\times 16$ non-overlapping patch.
\vspace{3pt}

\subsection*{Experiment A: small dataset and low dimensional dictionary}
We perform the same experiment as in \cite{LLLBJH15} to demonstrate the effectiveness of the proposed CS system $CS_{MT}$ without using the SRE $\bm E$. The training data is obtain by randomly chosen $6000$ samples from \emph{Training Data I} and the K-SVD algorithm is used to train the dictionary $\bm \Psi$.

Similar to \cite{LLLBJH15}, the parameters $M$, $N$, $L$ and $K$ are set to $20$, $64$, $100$ and $4$, respectively. The trade-off parameter $\lambda$ in $CS_{LH}$ is set to $0.1$ to yield a highest $\varrho_{psnr}$ for \emph{Testing Data I}. We also set $\lambda=0.1$ for the proposed CS system $CS_{MT}$.

\vspace{3pt}
The behavior of the five projection matrices in terms of mutual coherence and projection noise is examined and shown in Table \ref{t_1}. In order to illustrate the effectiveness of the proposed projection matrix, ten natural images are conducted to check its performance in terms of PSNR. The results are shown in Table~\ref{t_2}.

\vspace{3pt}
\noindent{\emph {Remark $3$:}}
\begin{itemize}
	\item As seen from Table \ref{t_1} , the results are self-explanatory. It shows that $CS_{MT}$ has small $\|\bm\Phi\|_F$ and also small projection noise $\|\bm\Phi\bm E\|_F$. This supports the proposed idea of using $\|\bm\Phi\|_F$ as a surrogate of $\|\bm\Phi\bm E\|_F$  to design the robust projection matrix.
	
	\item As shown in Table \ref{t_2}, we observe that $CS_{MT}$ outperforms $CS_{LH}$ in terms of $\varrho_{psnr}$ for most of the tested images.  We note that as long as an image can be approximately sparsely represented by the learned dictionary $\bm \Psi$, it is expected that the CS system $CS_{MT}$ yields reasonable performance for this image since the sensing matrix utilized in $CS_{MT}$ considers almost all the patterns of the SRE rather than a fixed one (as indicated by \eqref{eq:lem 2}) and thus is robust to SRE.

We also observe that $CS_{S-DCS}$ has highest $\varrho_{psnr}$; this is because $CS_{S-DCS}$ simultaneously optimizes the projection matrix and the sparsifying dictionary. It is of interest to note that $CS_{S-DCS}$ also has small $\|\bm\Phi\|_F$  and $\|\bm\Phi\bm E\|_F$ (as shown in Table \ref{t_1}). This again indicates that it is reasonable to minimize $\|\bm\Phi\|_F$ to get small projection noise $\|\bm\Phi\bm E\|_F$.
\end{itemize}
\begin{table*}[!htb]
	\centering   \caption{Performance Evaluated with Different Measures for Each of The Five Systems ($M=20$, $N=64$, $L=100$, $K=4$).}\label{t_1}
	\begin{tabular}{l||c|c|c|c|c}\hline\hline
		&$\|\bm I_L-\bm G\|_F^2$&$\mu(\bm D)$&$\mu_\text{av}(\bm D)$&$\|\bm \Phi\|_F^2$&$\|\bm\Phi\bm E\|_F^2$\\\hline
		$CS_{randn}$&$5.30\times10^5$&$0.951$&$0.384$&$1.25\times10^3$&$4.86\times10^3$\\\hline
		$CS_{DCS}$&$7.82\times10^4$&$0.999$&$0.695$&$3.41\times10^1$&$1.30\times10^2$\\\hline
		$CS_{S-DCS}$&$8.00\times10^1$&$0.857$&$0.326$&$6.24\times10^0$&$3.48\times10^1$\\\hline
		$CS_{LH}$&$8.02\times10^1$&$0.859$&$0.330$&$3.03\times10^7$&$3.07\times10^1$\\\hline
		$CS_{MT}$&$8.00\times10^1$&$0.848$&$0.331$&$6.59\times10^0$&$3.94\times10^1$
	\end{tabular}
\end{table*}


\begin{table*}[!htb]
	\centering
 \caption{Statistics of $\varrho_{psnr}$ for Ten Images Processed With $M=20$, $N=64$, $L=100$ for $K=4$. The highest $\varrho_{psnr}$ is marked in bold.}\label{t_2}
\begin{tabular}{l||c|c|c|c|c|c|c|c|c|c|c}\hline\hline
		&Lena&Elaine&Man&Barbara&Cameraman&Boat&Peppers&House&Bridge&Mandrill&Average\\\hline
		$CS_{randn}$&$29.01$&$29.23$&$27.65$&$22.46$&$23.15$&$26.57$&$24.71$&$28.34$&$26.49$&$20.61$&$25.82$\\\hline
		$CS_{DCS}$&$30.50$&$30.69$&$28.93$&$24.11$&$24.31$&$27.66$&$26.50$&$29.71$&$27.71$&$21.87$&$27.20$\\\hline
		$CS_{S-DCS}$&$\bm {33.18}$&$\bm{32.61}$&$\bm {31.52}$&$\bm{25.96}$&$\bm{26.75}$&$\bm{30.36}$&$\bm{29.71}$&$\bm{33.24}$&$\bm {30.20}$&$\bm{24.09}$&$\bm{29.76}$\\\hline
		$CS_{LH}$&$32.38$&$31.77$&$30.69$&$25.31$&$25.90$&$29.52$&$28.83$&$32.56$&$29.26$&$23.14$&$28.94$\\\hline
		$CS_{MT}$&$32.47$&$32.24$&$30.83$&$25.36$&$25.99$&$29.67$&$28.92$&$32.31$&$29.43$&$23.32$&$29.05$
	\end{tabular}
\end{table*}
\subsection*{Experiment B: large dataset and low dimensional dictionary}
In this set of experiments, we first learn a dictionary on large-scale training samples, \ie, \emph{Training Data I}, and then design the projection matrices with the learned dictionary. As discussed in the previous section, the large-scale training dataset makes it inefficient or even impossible to compute the SRE matrix $\bm E$. Therefore, it is inefficient to utilize the methods in \cite{LLLBJH15} \cite{HBLZ16} as they require  the SRE matrix $\bm E$. Similar reason holds for $CS_{S-DCS}$.
Fortunately, the following results show that the proposed CS system $CS_{MT}$ performs comparably to $CS_{LH}$.

The online dictionary learning algorithm in \cite{MBPS09} \cite{MBPS10} is chosen to train the sparsifying dictionary on the whole \emph{Training Data I}. For a fair comparison, we calculate the SRE $\bm E$ off-line for $CS_{LH}$ in this experiment.\footnote{In order to compare with $CS_{LH}$,
	we still compute the SRE matrix $\bm E$ for the training data though it requires abundant of extra storage and computation resources.} The same $M$, $N$, $L$, $K$ in \emph{Experiment A} are used in this experiment. 
$\lambda=0.9$ and $\lambda=1e-3$ are selected for $CS_{MT}$ and $CS_{LH}$, respectively. We note that the choice of $\lambda$ for $CS_{LH}$ is very sensitive to $\bm E$. This is because the two terms $\|\bm I_L-\bm \Psi^\mathcal T\tilde{\bm \Phi}^\mathcal T\tilde{\bm \Phi}\bm\Psi\|_F^2$ and $\|\tilde{\bm \Phi}\bm E\|_F^2$ in \eqref{e_7} for $CS_{LH}$  have different physical meanings and more importantly, the second term $\|\tilde{\bm \Phi}\bm E\|_F^2$ increases when we have more number of training data, while the first term is independent of the training data. Thus, we need to decrease $\lambda$ for $CS_{LH}$ when we increase the number of training data.

\vspace{3pt}
\noindent{\emph {Remark $4$:}}
\begin{itemize}
	\item  As shown in Table \ref{t_3}, benefiting from large-scale training samples, the performance of both $CS_{LH}$ and $CS_{MT}$ has been improved compared with the one in Table~\ref{t_2}. Moreover, we also observe that $CS_{MT}$ performs similarly to $CS_{LH}$. It is also of interest to note that the PSNR for $CS_{MT}$ in Table \ref{t_3} is higher than the one for $CS_{S-DCS}$ in Table \ref{t_2} for most of the tested images. This suggests that if the dictionary and the projection matrix are simultaneously optimized by online algorithm with large dataset, the performance of the corresponding CS system can be further improved since joint optimization ($CS_{S-DCS}$) is expected to have better performance than only optimizing projection matrix with a given dictionary ($CS_{MT}$) under the same settings. We note that the proposed framework for projection matrix design can be utilized for online simultaneous optimization of the dictionary and the projection matrix. Investigation along this direction is on-going.
	\item We compare the computational complexity of our proposed method with the one in \cite{LLLBJH15} \cite{HBLZ16}. The later mainly consists of two more steps: the calculations of the SRE matrix $\bm E$ and $\bm E\bm E^\mathcal T$. Calculating $\bm E$ involves the OMP algorithm \cite{RZE08} with computational complexity of $\mathcal O\left(PKNL(KL\log L+K^3)\right)$, where we repeat that $P$, $N$, $L$ and $K$ denote the number of samples, the dimension of signal, the number of atoms in dictionary and the sparsity level, respectively. The complexity for calculating $\bm E\bm E^\mathcal T$ is $\mathcal O\left(PN^2\right)$. Thus, compared with $CS_{MT}$, $CS_{LH}$ needs at least more computational time of $\mathcal O\left(PNK^2L^2\log L+PNLK^4+PN^2\right)$. In the set of next experiments, we will show the advantage of designing the projection matrix on a high dimensional dictionary. With $N$ and $L$ increasing, the efficiency of the proposed method $CS_{MT}$ becomes more distinct.
\end{itemize}


\begin{table*}[!htb]
	\centering   \caption{Statistics of $\varrho_{psnr}$ for Ten Images Processed With $M=20$, $N=64$, $L=100$ for $K=4$. The dictionary is trained on a large dataset. The highest $\varrho_{psnr}$ is marked in bold.}\label{t_3}
	\begin{tabular}{l||c|c|c|c|c|c|c|c|c|c|c}\hline\hline
		&Lena&Elaine&Man&Barbara&Cameraman&Boat&Peppers&House&Bridge&Mandrill&Average\\\hline
		$CS_{randn}$&$30.54$&$29.88$&$28.69$&$22.52$&$23.94$&$27.48$&$26.75$&$30.31$&$27.21$&$20.93$&$26.83$\\\hline
		$CS_{DCS}$&$30.20$&$29.91$&$28.57$&$23.43$&$23.89$&$27.33$&$26.51$&$30.06$&$27.31$&$21.36$&$26.86$\\\hline
		$CS_{S-DCS}$&$-$&$-$&$-$&$-$&$-$&$-$&$-$&$-$&$-$&$-$&$-$\\\hline
		$CS_{LH}$&$\bm{33.92}$&$32.78$&$\bm{31.88}$&$\bm{25.73}$&$26.85$&$30.60$&$\bm{30.17}$&$32.24$&$30.13$&$23.87$&$29.82$\\\hline
		$CS_{MT}$&$33.91$&$\bm{32.81}$&$\bm{31.88}$&$25.71$&$\bm{26.87}$&$\bm{30.62}$&$30.14$&$\bm{33.54}$&$\bm{30.16}$&$\bm{23.88}$&$\bm{29.95}$
	\end{tabular}
\end{table*}


	\begin{table*}[!htb]
	\centering   \caption{Statistics of $\varrho_{psnr}$ for Ten Images Processed With $M=80$, $N=256$, $L=800$ for $K=16$.  The highest $\varrho_{psnr}$ is marked in bold.}\label{t_4}
	\begin{tabular}{l||c|c|c|c|c|c|c|c|c|c|c}\hline
		&Lena&Elaine&Man&Barbara&Cameraman&Boat&Peppers&House&Bridge&Mandrill&Average\\\hline
		$CS_{randn}$&$30.74$&$29.70$&$28.82$&$22.76$&$23.65$&$27.38$&$27.14$&$30.77$&$26.95$&$20.77$&$26.87$\\\hline
		$CS_{DCS}$&$29.82$&$29.09$&$27.94$&$22.99$&$23.15$&$26.56$&$25.90$&$29.51$&$26.57$&$20.85$&$26.24$\\\hline
		$CS_{S-DCS}$&$-$&$-$&$-$&$-$&$-$&$-$&$-$&$-$&$-$&$-$&$-$\\\hline
		$CS_{LH}$&$-$&$-$&$-$&$-$&$-$&$-$&$-$&$-$&$-$&$-$&$-$\\\hline
		$CS_{MT}$&\bm{$34.41}$&$\bm{32.96}$&$\bm{32.35}$&$\bm{26.02}$&$\bm{27.18}$&$\bm{30.97}$&$\bm{30.74}$&$\bm{34.33}$&$\bm{30.31}$&$\bm{24.01}$&$\bm{30.33}$
	\end{tabular}
\end{table*}

\subsection*{Experiment C: large dataset and high dimensional dictionary}
Inspired by the work in \cite{SOZE16}, we attempt to design the projection matrix on a high dimensional dictionary in this set of experiments.  The reason to utilize a high dimensional dictionary is as follows. The sparse representation of a natural image $\bm X$ can be written as,
\en
\begin{array}{rcl}
	\bm X &=& \tilde{\bm X}+\bm E\\
	\tilde{\bm X}&\triangleq&\bm \Psi\bm\Theta
\end{array}
\een
where $\bm E$ is the sparse representation error.\footnote{Since OMP is used to conduct the sparse coding mission in this paper, each column of $\bm\Theta$ is exactly $K$-sparse.} We first recover $\bm\Theta$ by solving a set of \eqref{e_3} and then take $\tilde{\bm X}=\bm \Psi\bm\Theta$ as the recovered image. It is clear that no matter what projection matrix is utilized, the best we can obtain is $\tilde{\bm X}$ instead of $\bm X$. Thus, with a dictionary which can capture more information of the training dataset and better represent $\bm X$ with $\tilde{\bm X}$,  the corresponding CS system is excepted to yield a higher SRA. As stated in \cite{SOZE16}, training the dictionary with larger patches  results in smaller sparse representation errors for natural images. However, training dictionary on larger patches, we have to train on a large-scale dataset to better represent the signals of interest. This demonstrates the efficiency of the proposed method for designing a robust projection matrix on  a high dimensional dictionary as this method drops the requirement of the SRE matrix $\bm E$ which is not only in high dimension, but also large-scale.


The parameters $M$, $N$, $L$, $K$ and $\lambda$ are set to $80$, $256$, $800$, $16$ and $0.5$, respectively. Due to the fact that $CS_{MT}$ has a similar performance with $CS_{LH}$ and the choice of $\lambda$ for $CS_{LH}$  is very sensitive to $\bm E$, we omit the performance of $CS_{LH}$ in this experiment. The simulation results are presented in Table \ref{t_4}.  In order to demonstrate the visual effect clearly, two images ‘Lena’ and ‘Mandrill’ are shown in Figs. \ref{Lena_recoved} and \ref{Mandrill_recoved}, respectively. For a clear comparison, We choose the projection matrices and corresponding dictionary which yields the highest average $\varrho_{psnr}$ from Table \ref{t_2} to Table \ref{t_4} in Figs. \ref{Lena_recoved} and \ref{Mandrill_recoved}.

\begin{figure*}
	\centering
	\subfigure[]{\includegraphics[width=0.3\textwidth]{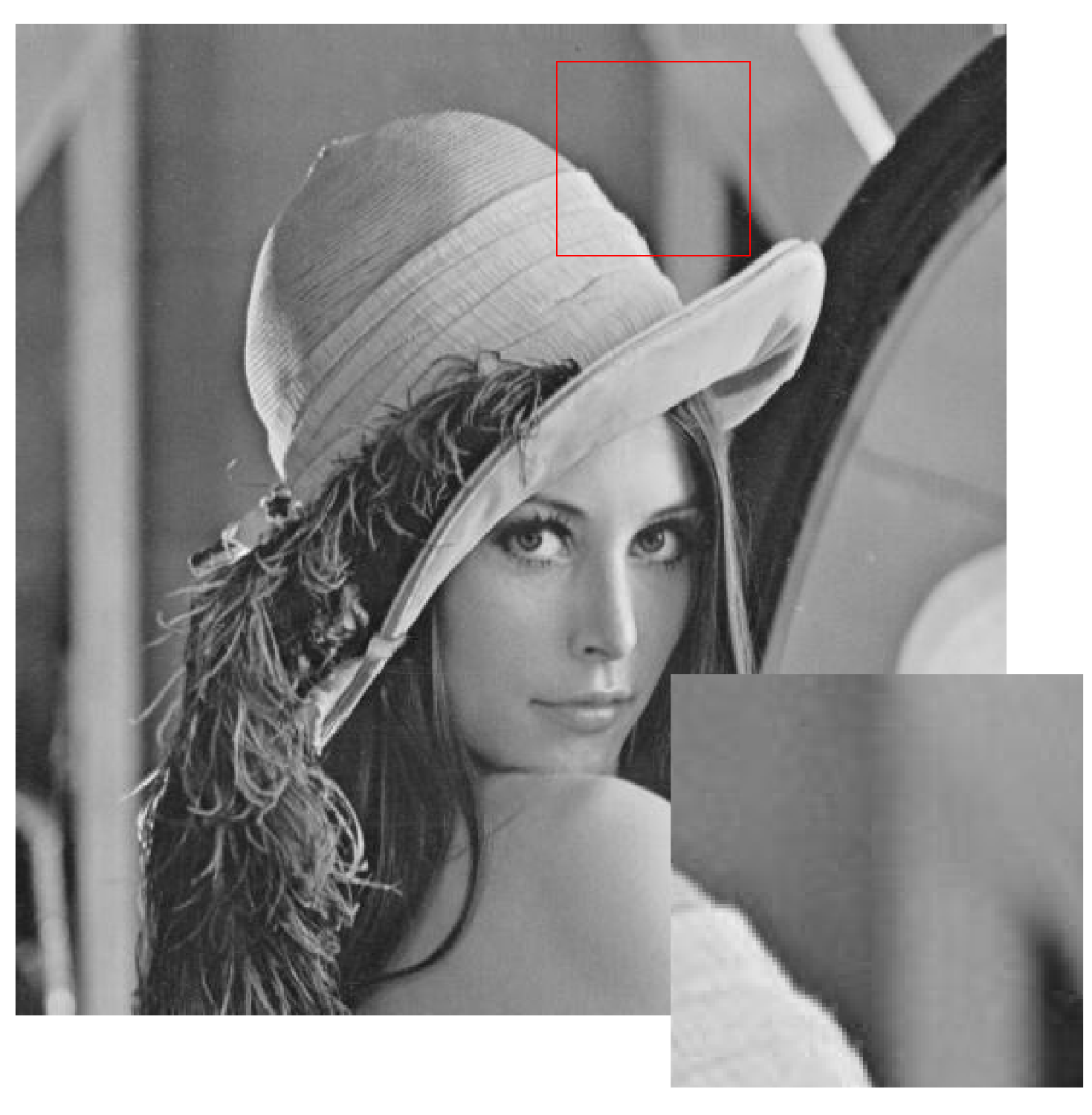}}
	\subfigure[]{\includegraphics[width=0.3\textwidth]{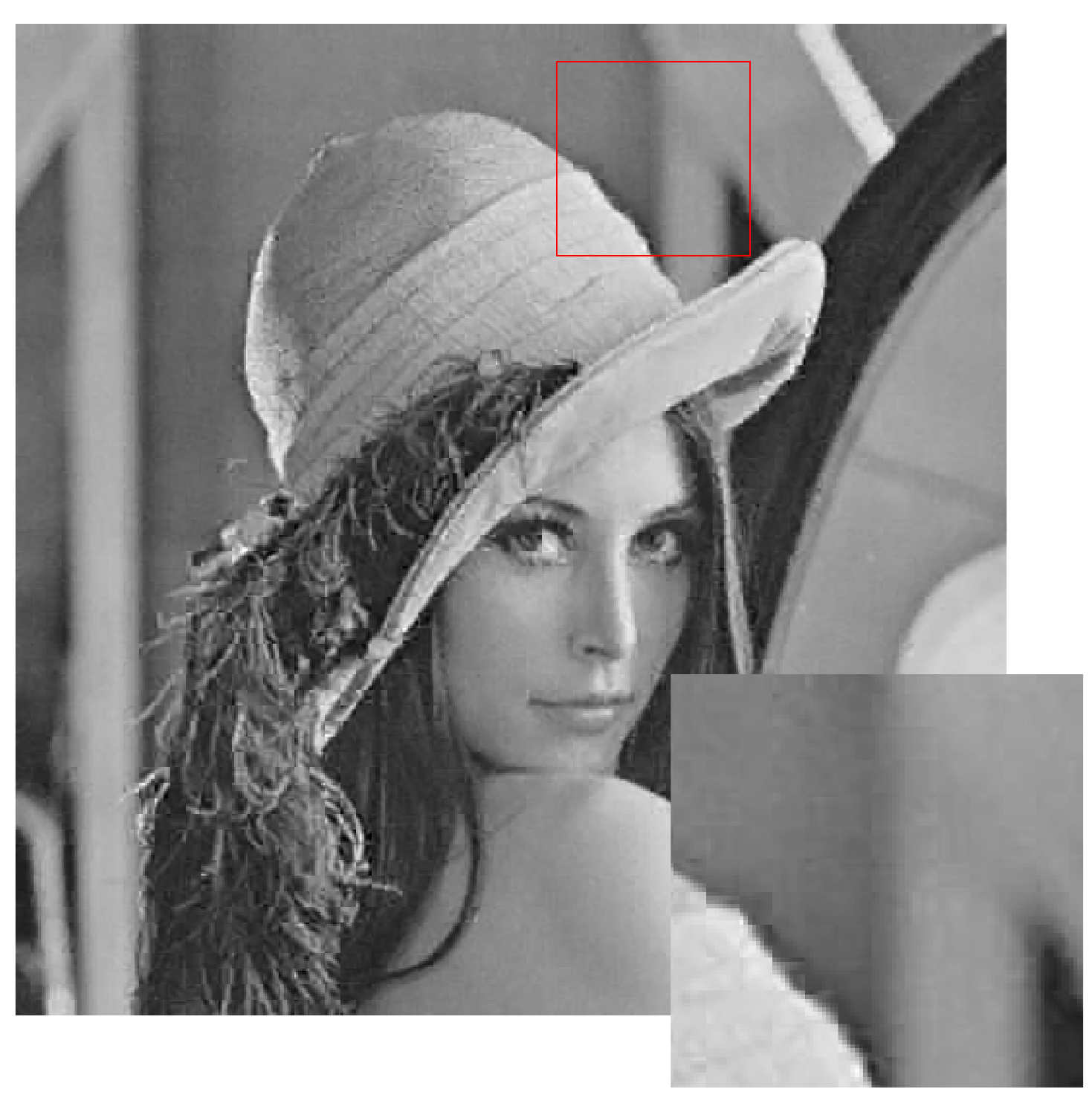}}
	\subfigure[]{\includegraphics[width=0.3\textwidth]{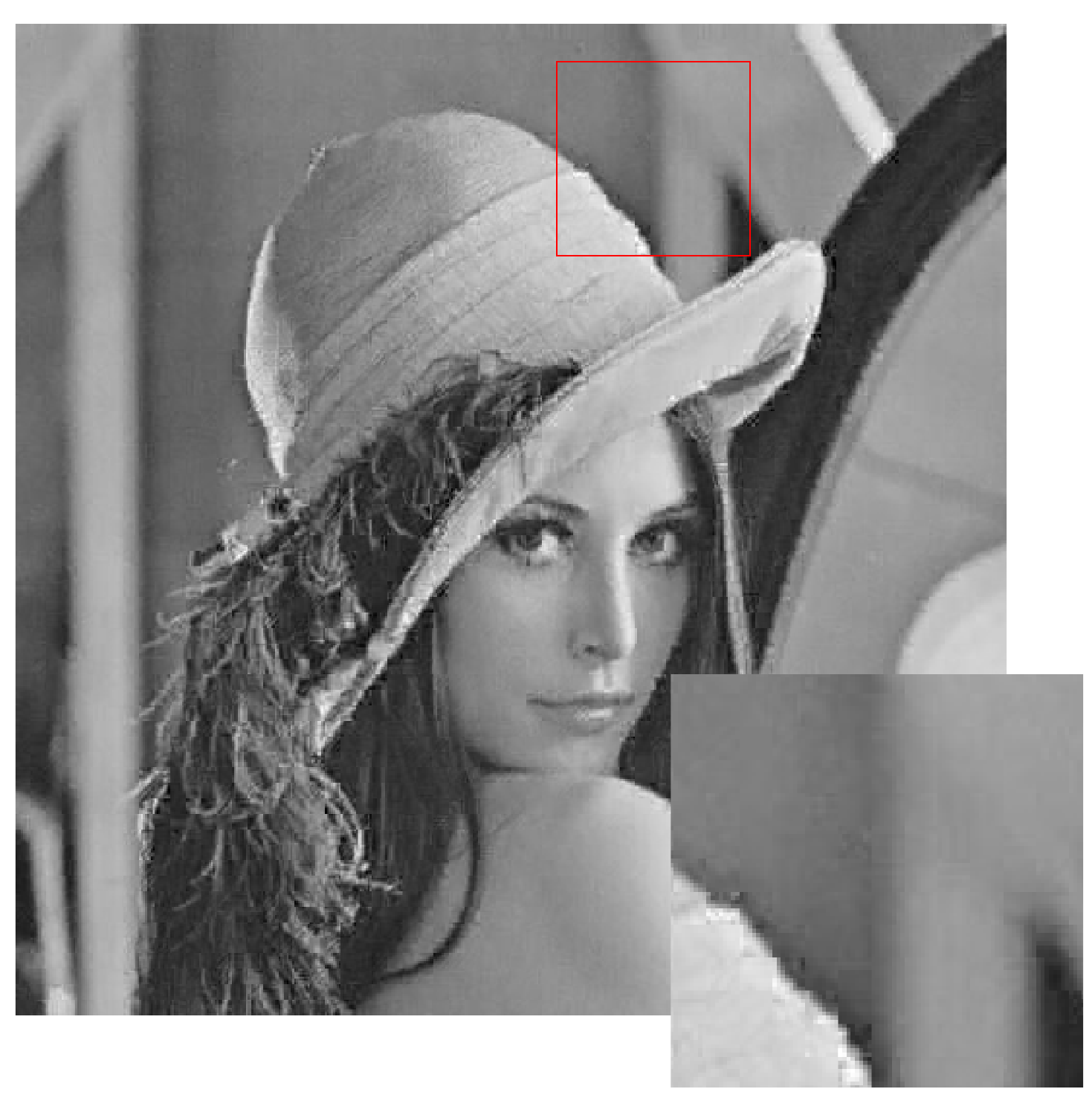}}\\

	\subfigure[]{\includegraphics[width=0.3\textwidth]{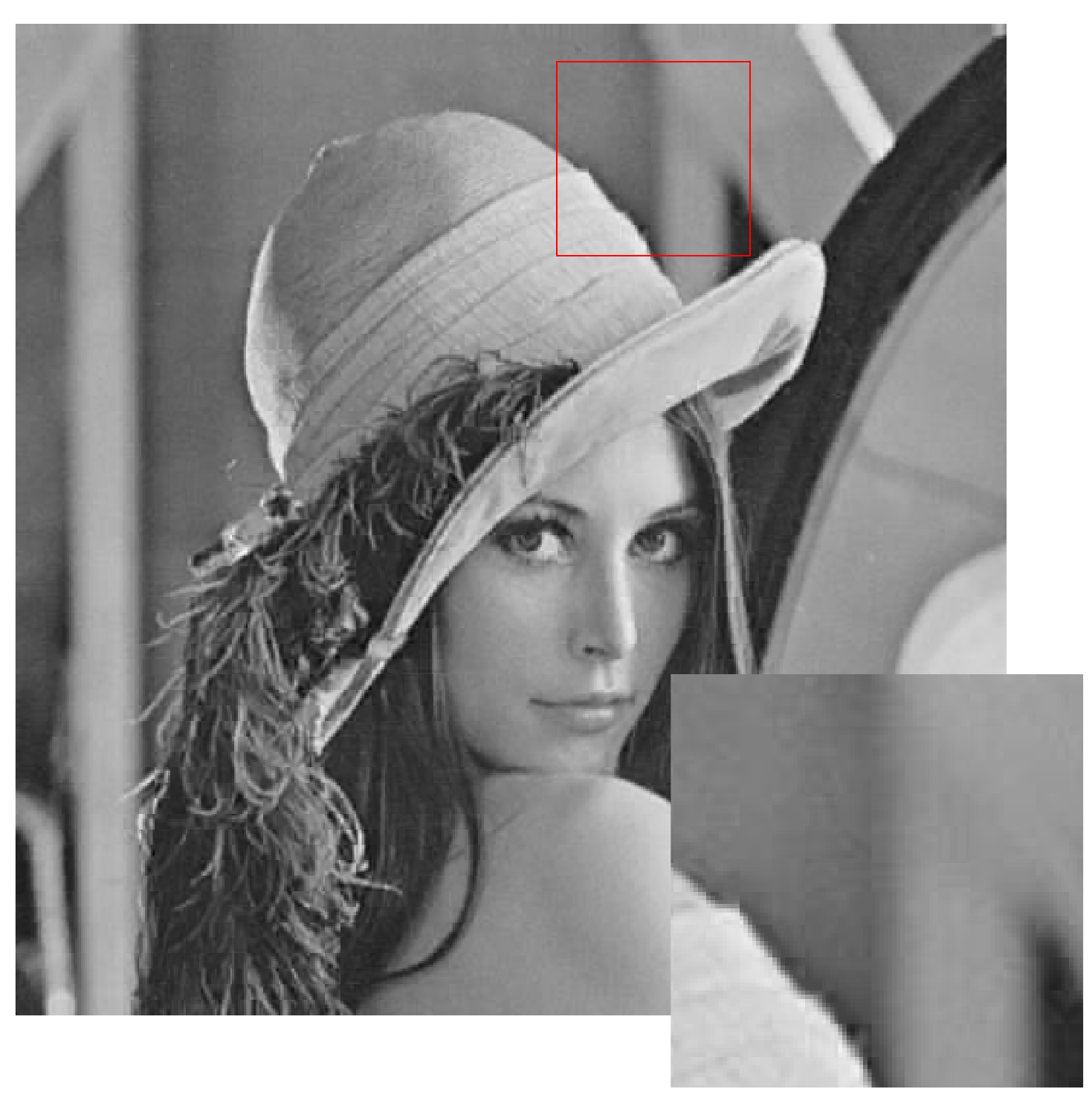}}
	\subfigure[]{\includegraphics[width=0.3\textwidth]{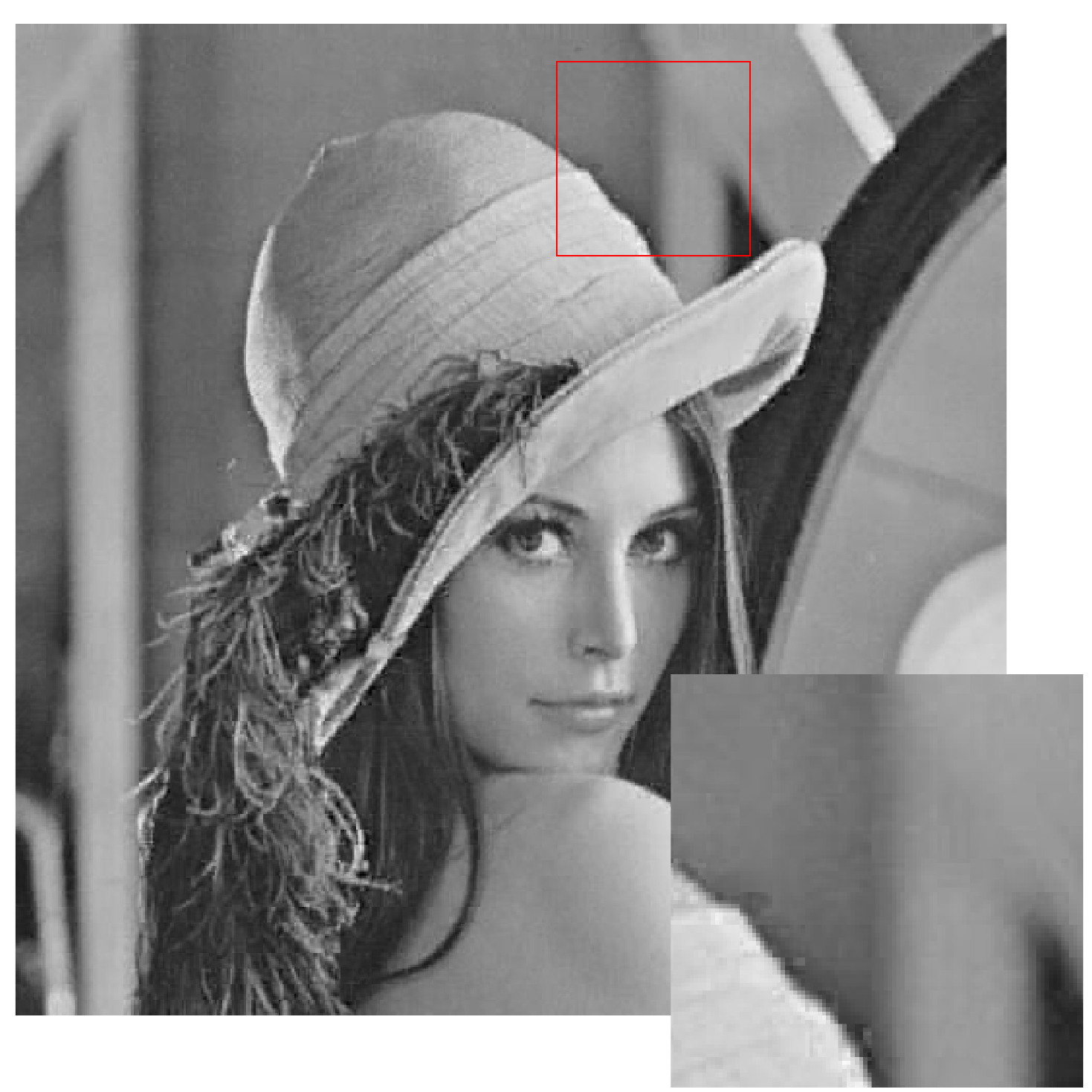}}
	\subfigure[]{\includegraphics[width=0.3\textwidth]{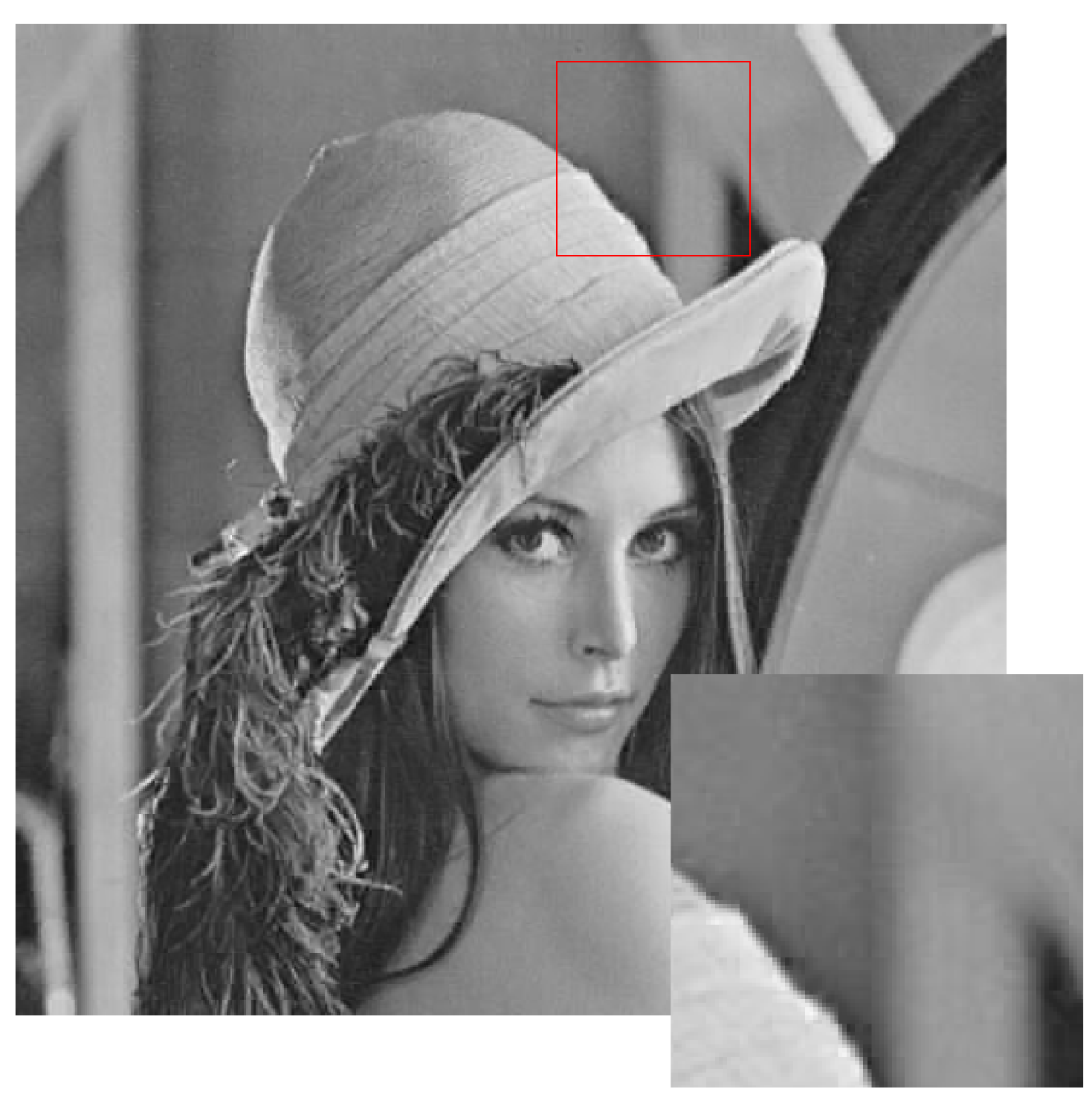}}\\
	
	\subfigure[]{\includegraphics[width=0.3\textwidth]{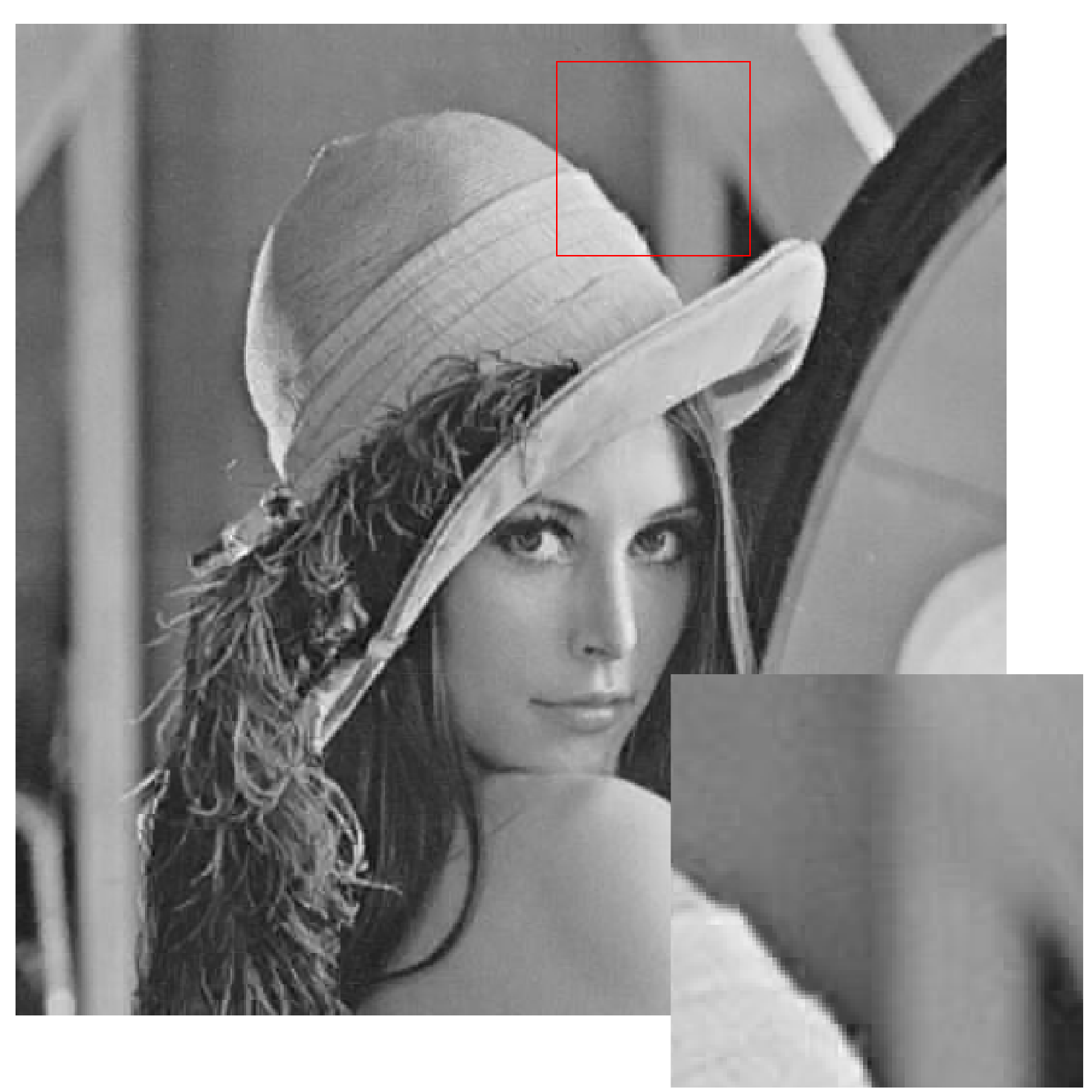}}
	\subfigure[]{\includegraphics[width=0.3\textwidth]{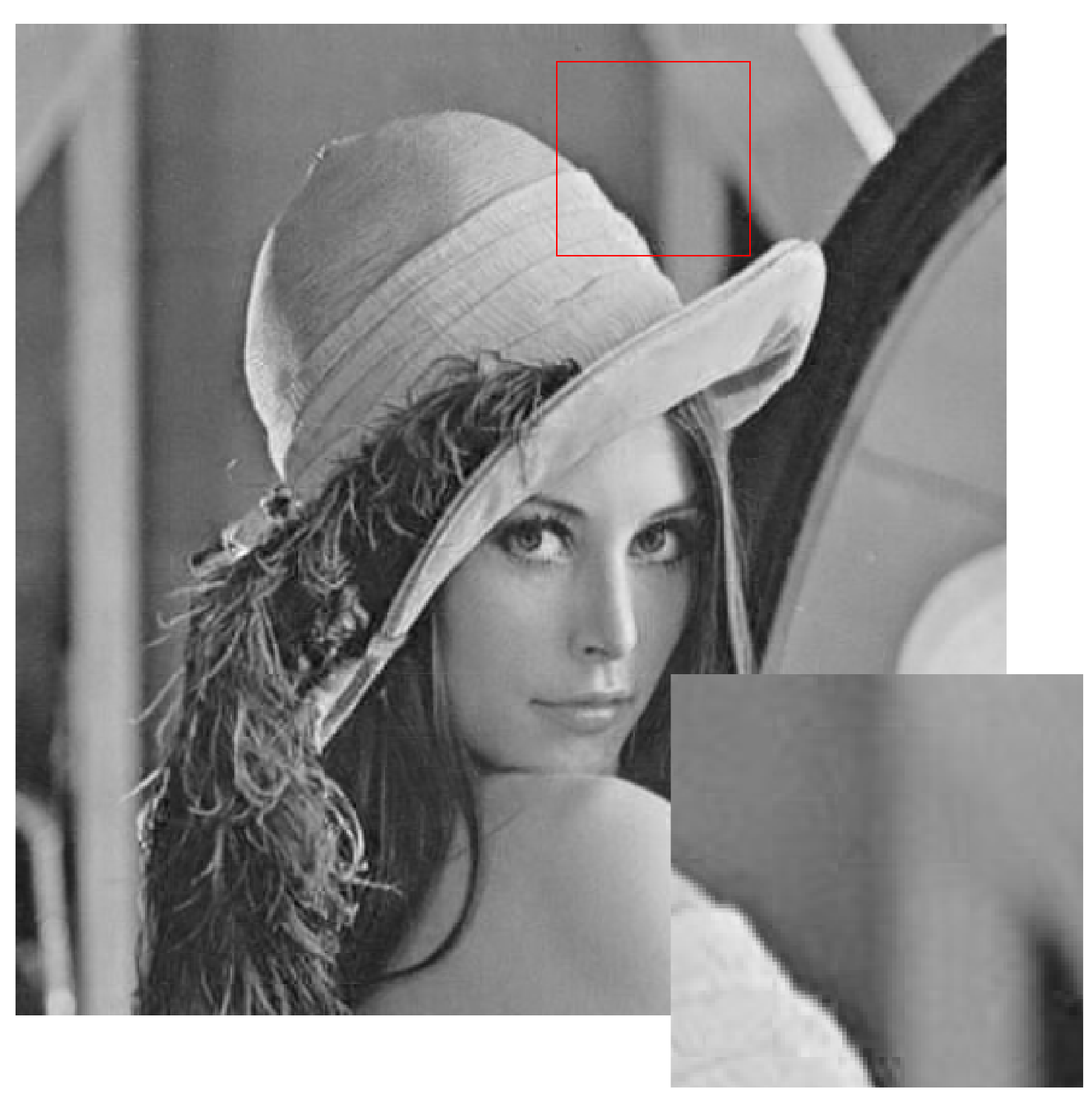}}\\
	
	\caption{‘Lena’ and its reconstructed images from the corresponding CS systems. (a)The original. (b) $CS_{randn}$ (Experiment B). (c) $CS_{DCS}$ (Experiment A). (d) $CS_{LH}$ (Experiment B). (e) $CS_{S-DCS}$ (Experiment A). (f) - (h) $CS_{MT}$ (From Experiment A to Experiment C). The corresponding $\varrho_{psnr}$ can be found from Table \ref{t_2} to Table \ref{t_4}.}\label{Lena_recoved}
\end{figure*}

\begin{figure*}
	\centering
	\subfigure[]{\includegraphics[width=0.3\textwidth]{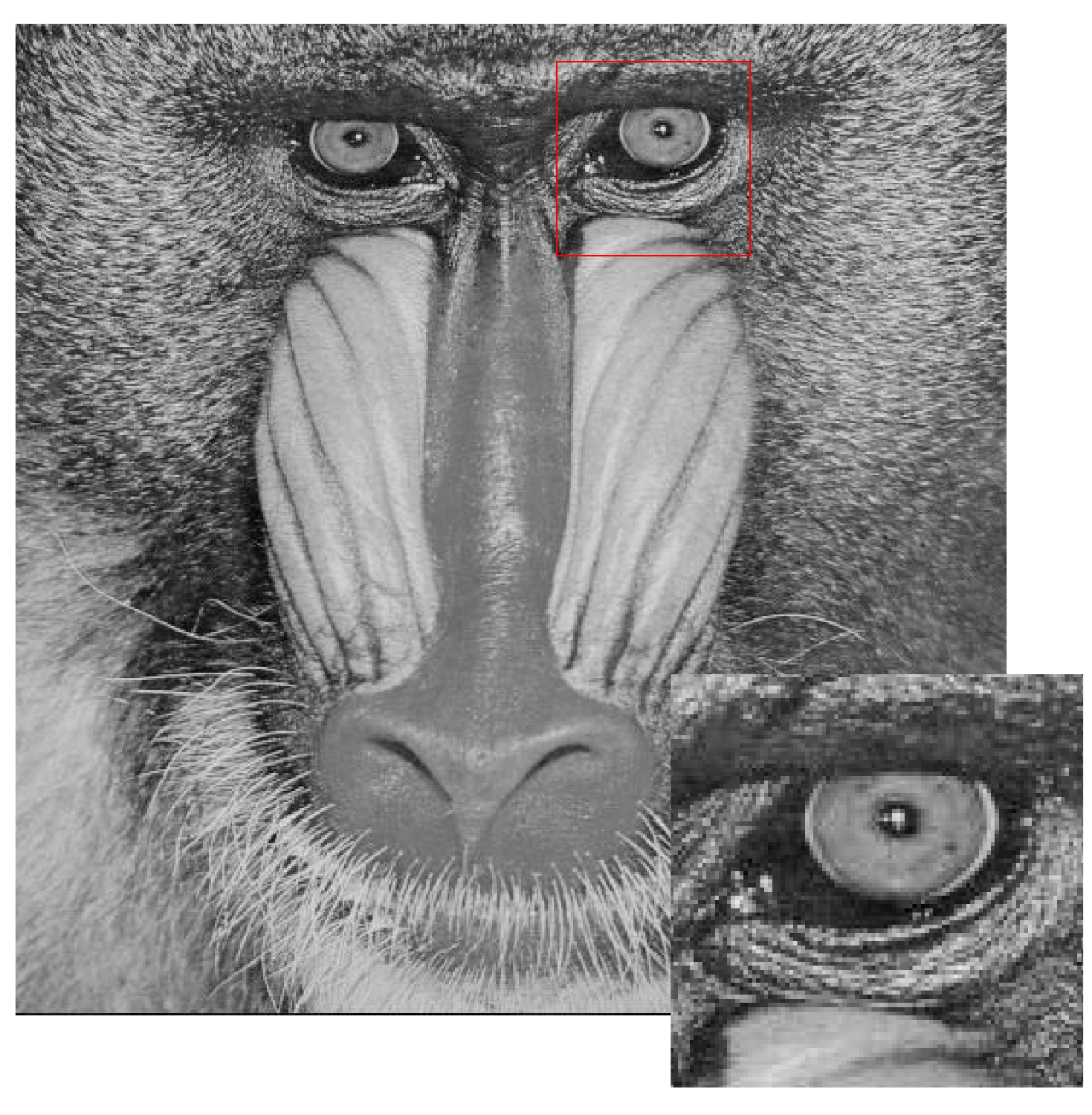}}
	\subfigure[]{\includegraphics[width=0.3\textwidth]{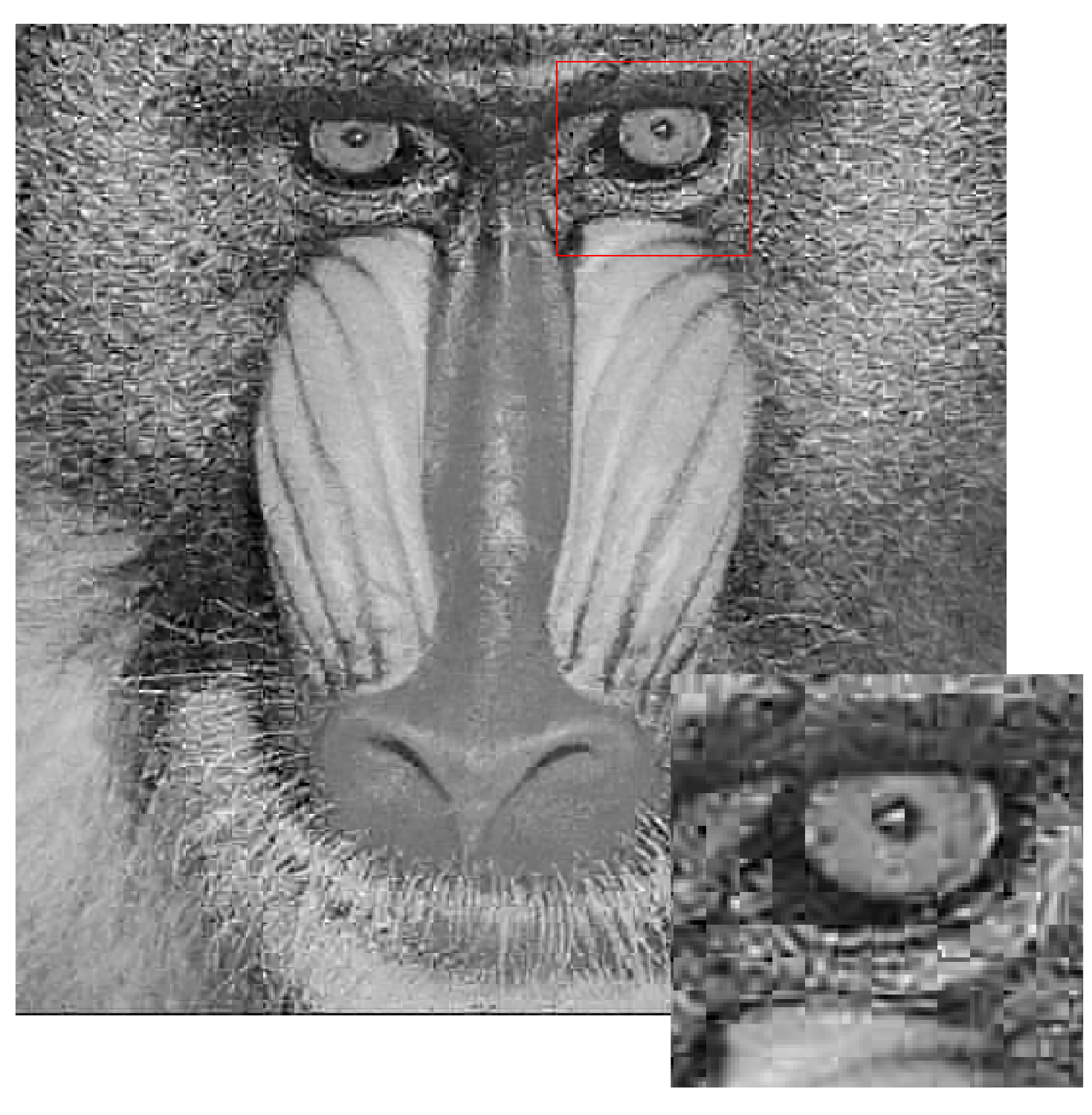}}	
	\subfigure[]{\includegraphics[width=0.3\textwidth]{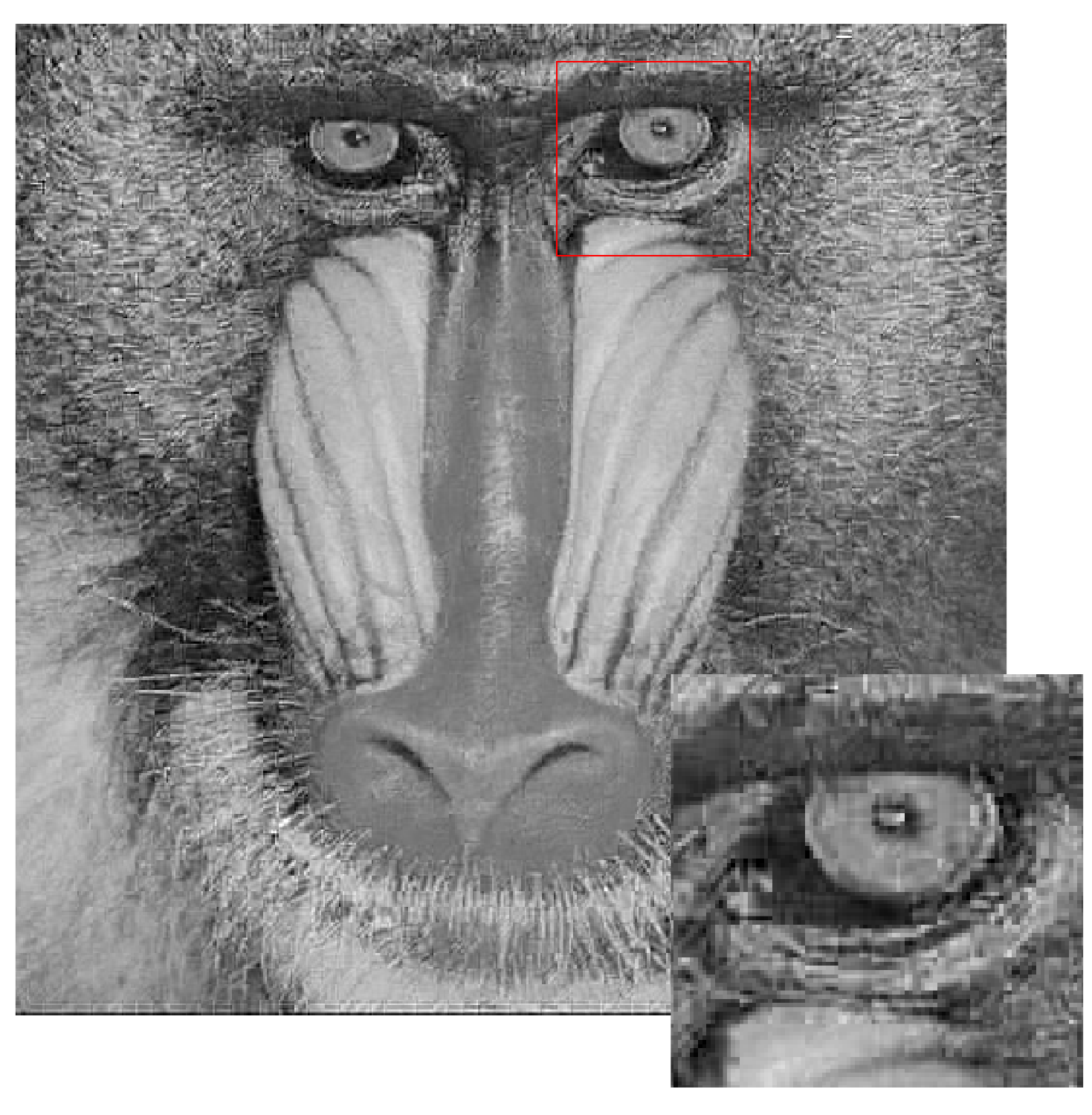}}\\
	
	\subfigure[]{\includegraphics[width=0.3\textwidth]{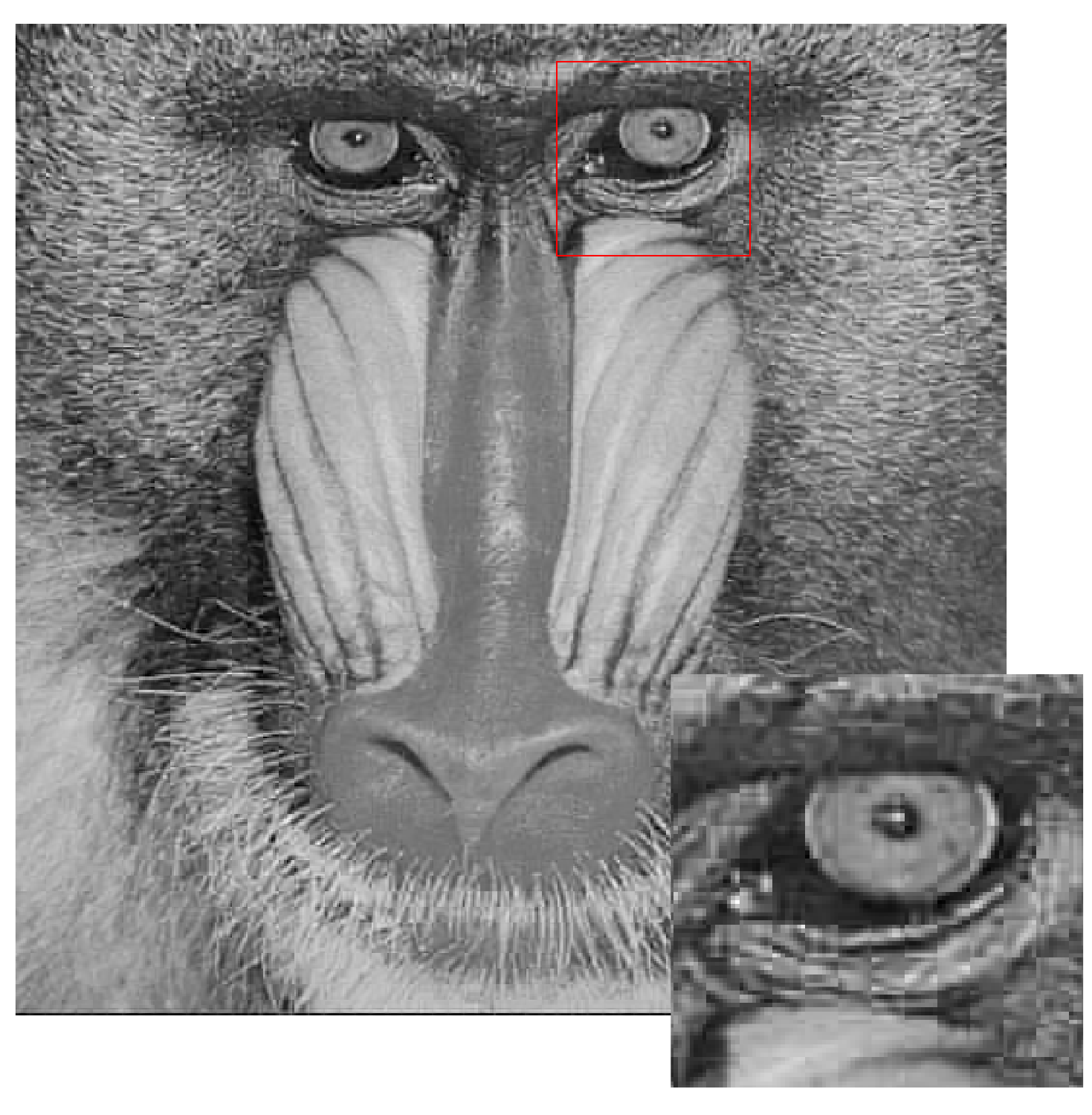}}
	\subfigure[]{\includegraphics[width=0.3\textwidth]{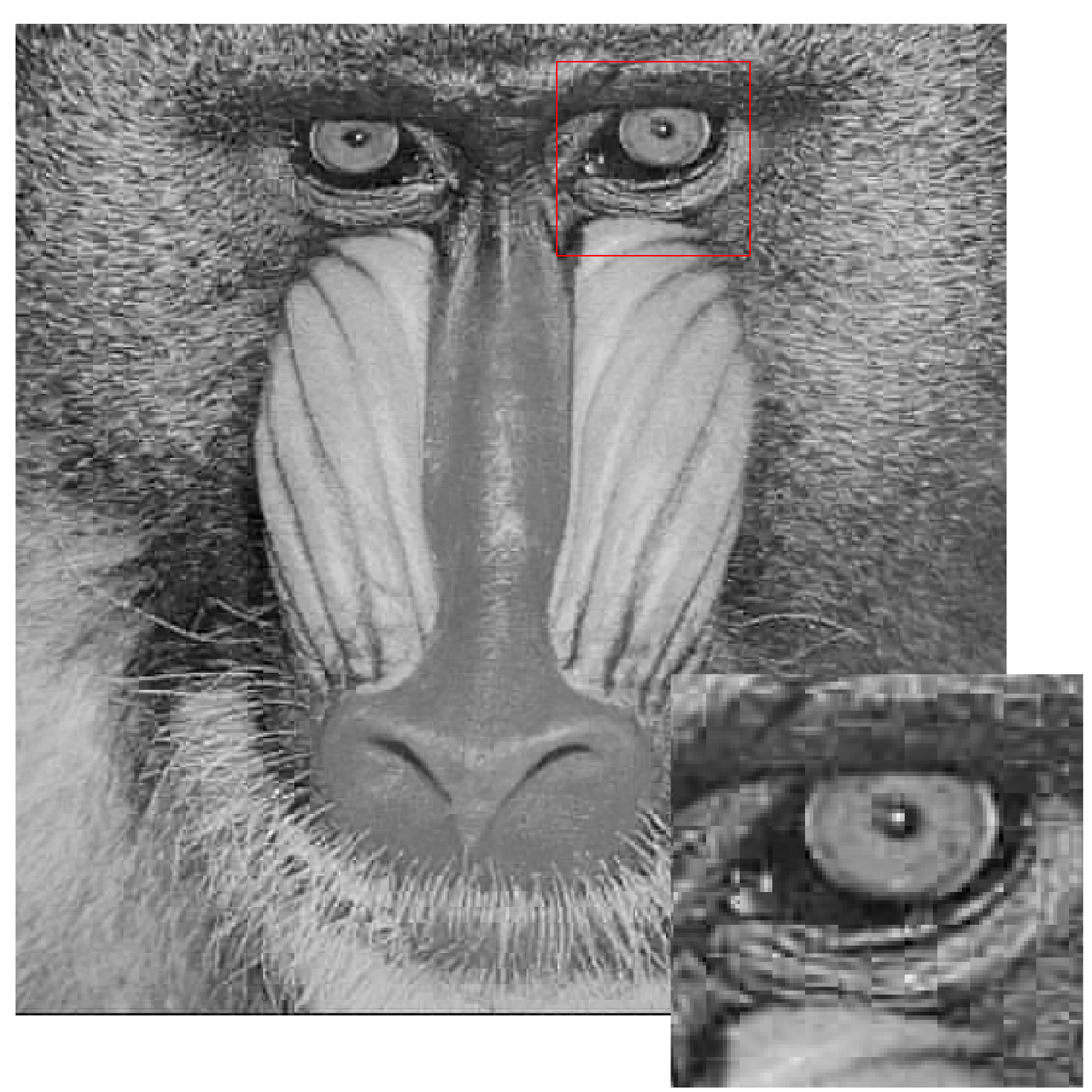}}\\
	
	\subfigure[]{\includegraphics[width=0.3\textwidth]{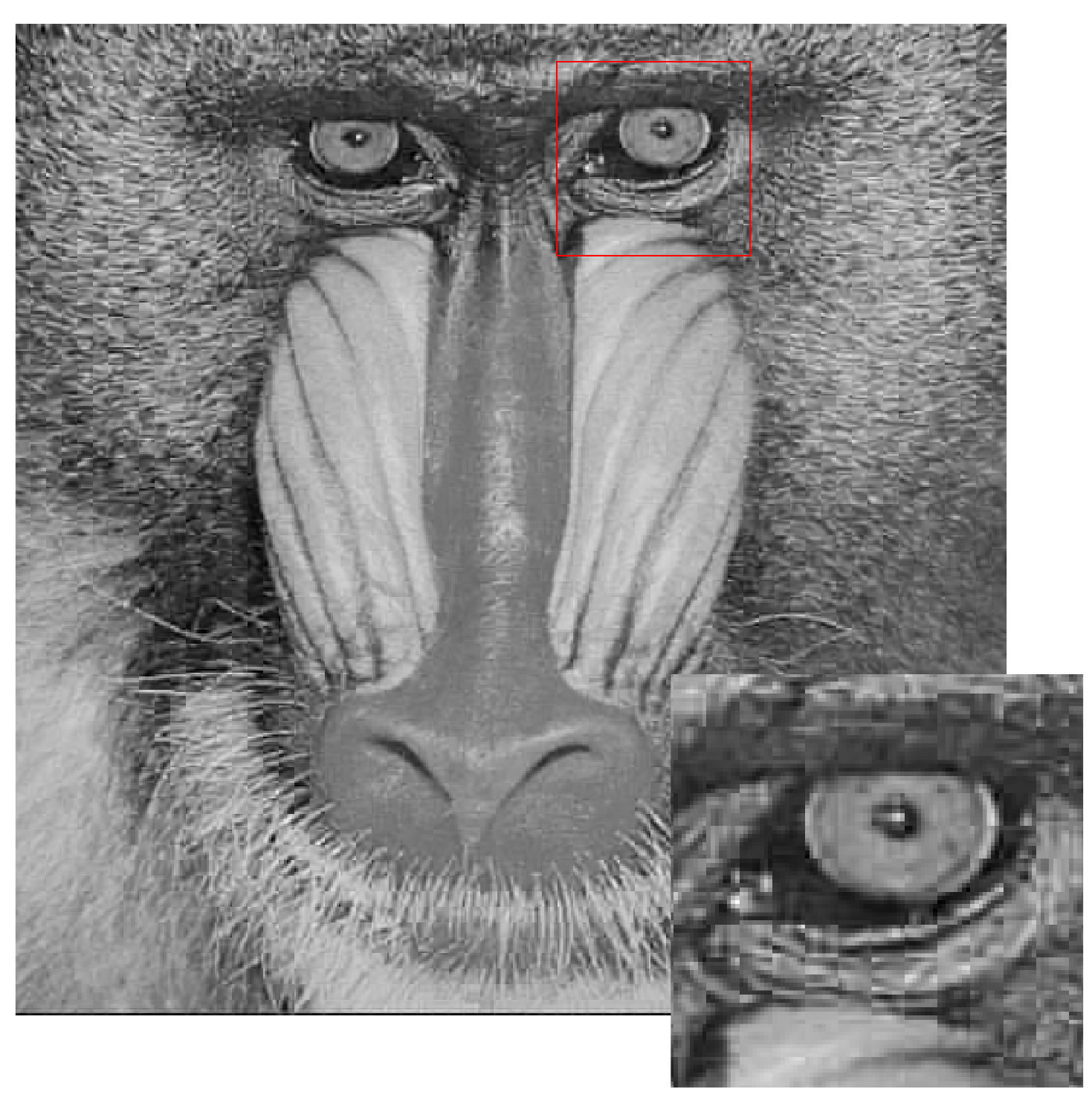}}
	\subfigure[]{\includegraphics[width=0.3\textwidth]{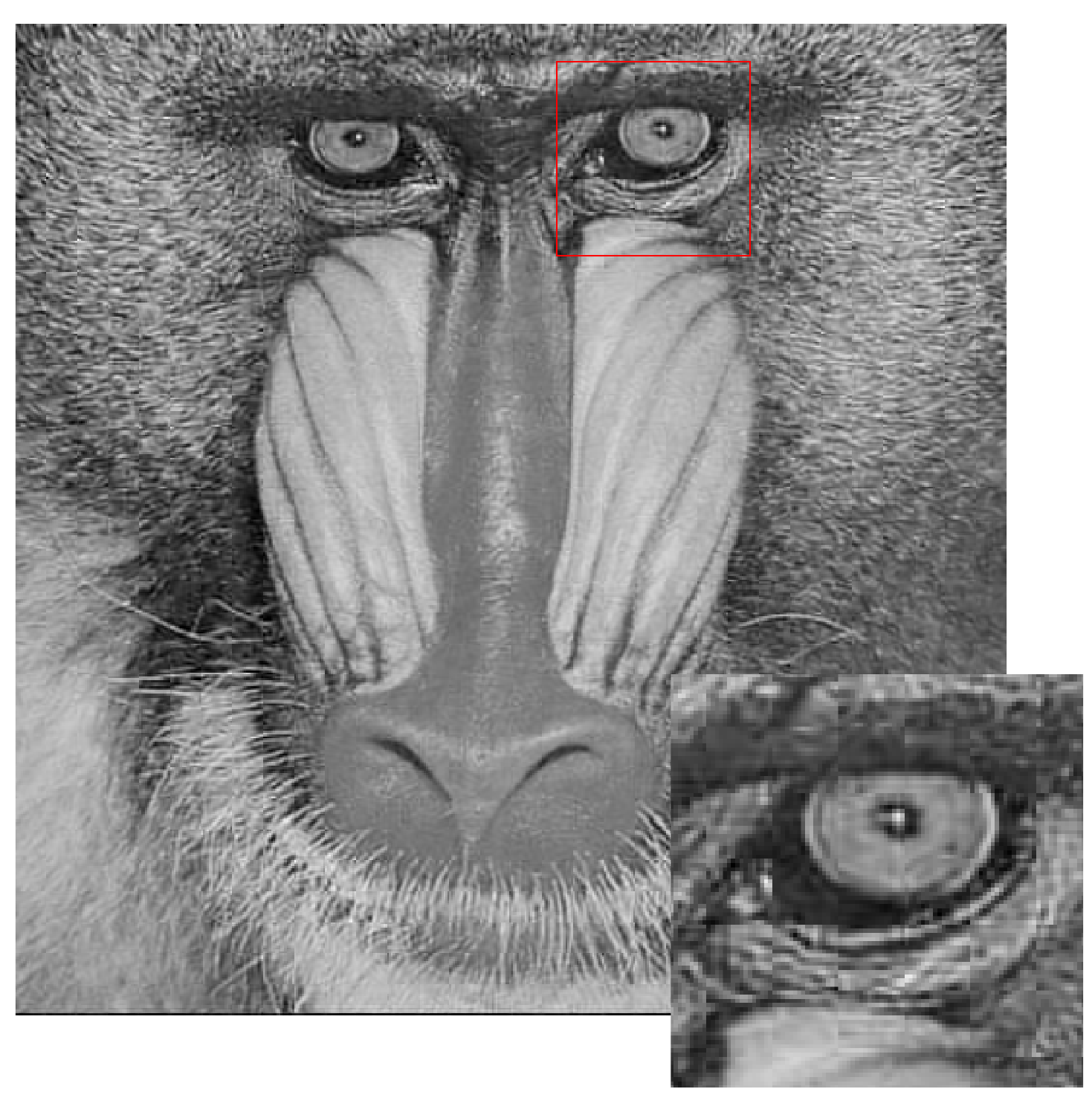}}\\
	
	\caption{‘Mandrill’ and its reconstructed images from the corresponding CS systems. (a)The original. (b) $CS_{randn}$ (Experiment B). (c) $CS_{DCS}$ (Experiment A). (d) $CS_{LH}$ (Experiment B). (e) $CS_{S-DCS}$ (Experiment A). (f) - (h) $CS_{MT}$ (From Experiment A to Experiment C). The corresponding $\varrho_{psnr}$ can be found from Table \ref{t_2} to Table \ref{t_4}.}\label{Mandrill_recoved}
\end{figure*}

\noindent \emph{Remark $5$:}
\begin{itemize}
	\item We observe that for the CS systems $CS_{randn}$ and $CS_{DCS}$,  designing the projection matrix on a high dimensional dictionary results in similar performance to  what is shown in \emph{Experiments A} and \emph{B} with the same compression rate $\frac{M}{N}$. Moreover, $CS_{DCS}$  has lower $\varrho_{psnr}$ in \emph{Experiments B} and \emph{C} than \emph{Experiment A}. However, the proposed CS system $CS_{MT}$ has increasing $\varrho_{psnr}$ from  \emph{Experiment A} to \emph{Experiment C}. This indicates the effectiveness of the proposed method for a CS system with a higher dimensional dictionary.
	
	\item We also investigate the influence of the parameters $M$, $L$ , $K$ to the above mentioned CS systems. The simulation results on \emph{Testing Data II} are given in Fig.s \ref{Vary_M_High} to \ref{Vary_L_High}. As can be observed, the proposed CS system $CS_{MT}$ has highest $\varrho_{psnr}$ among the three CS systems.
	
	\item The recent work in \cite{AJBG16} states that it is possible to train the dictionary on millions of training signals whose dimension is also more than one million. The proposed method can be utilized to design a robust projection matrix on such high dimensional dictionaries since it gets rid of the requirement of the SRE matrix $\bm E$. Note that in this case, more efforts for efficiently solving \eqref{e_9} are needed. A full investigation regarding this direction belongs to a future work. 
\end{itemize}

%

Three sets of experiments on natural images are conducted to illustrate the effectiveness and efficiency of the proposed framework in Section \ref{S_3}. A dictionary trained on a larger dataset can better represent the signal and the corresponding CS system yields better performance in terms of SRA. Additionally, a high dimensional dictionary has more freedom to represent the signals of interest. The CS system with a high dimensional dictionary and a projection matrix obtained by the proposed method results in higher $\varrho_{psnr}$. However, both cases need to train the dictionary on a large-scale training dataset, making it  inefficient or even impossible for computing the SRE matrix $\bm E$. One of the main contributions in this paper is proposing a new framework that is independent of the SRE matrix $\bm E$.

\begin{figure}[htb!]
	\centering
	\includegraphics[width=0.5\textwidth]{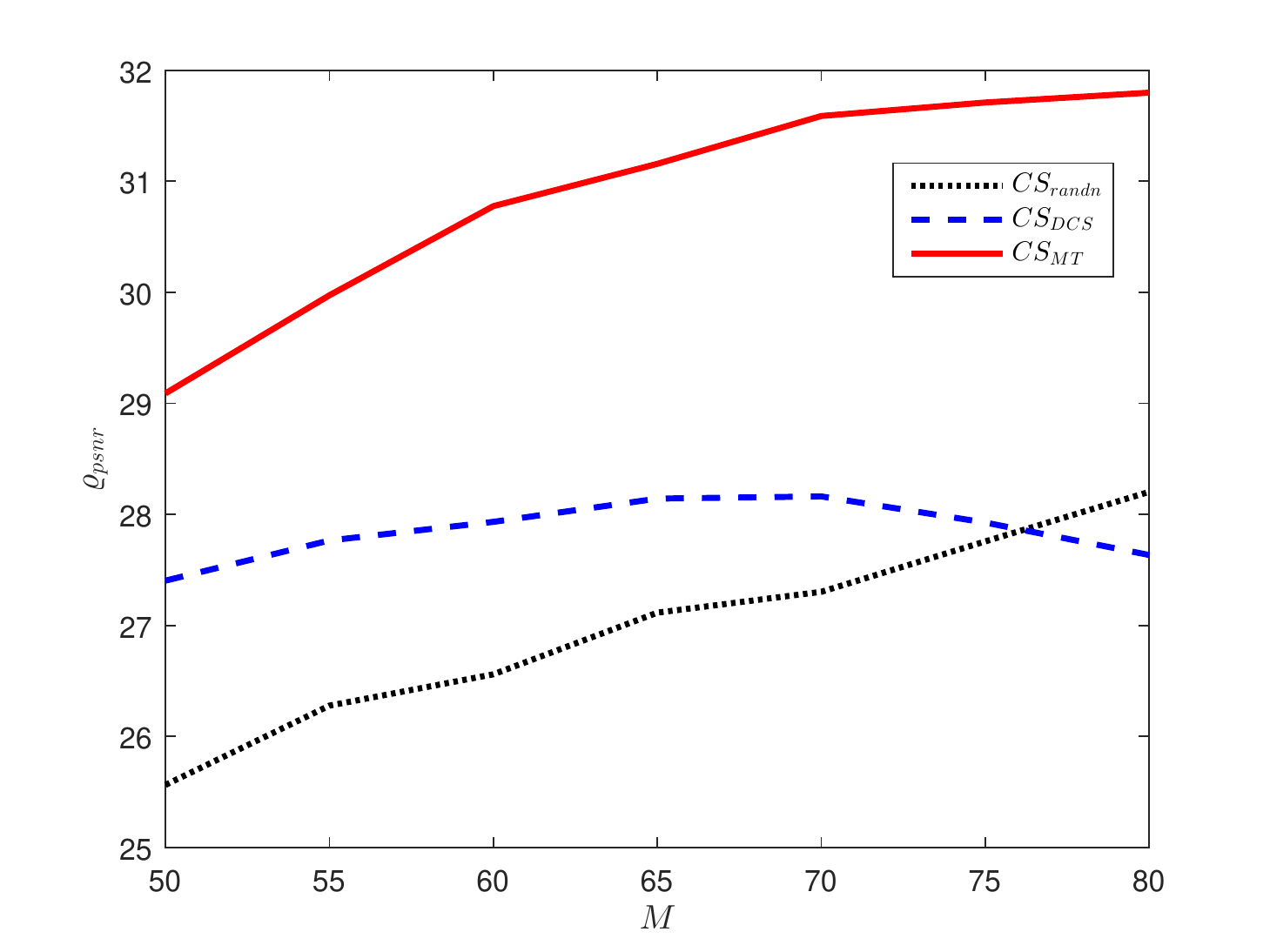}
	\caption{The statistic $\varrho_{psnr}$ versus the varying of $M$ on \emph{Testing Data II} for a fixed $N=256$, $L=800$ and $K = 16$.}\label{Vary_M_High}
\end{figure}
\begin{figure}[htb!]
	\centering
	\includegraphics[width=0.5\textwidth]{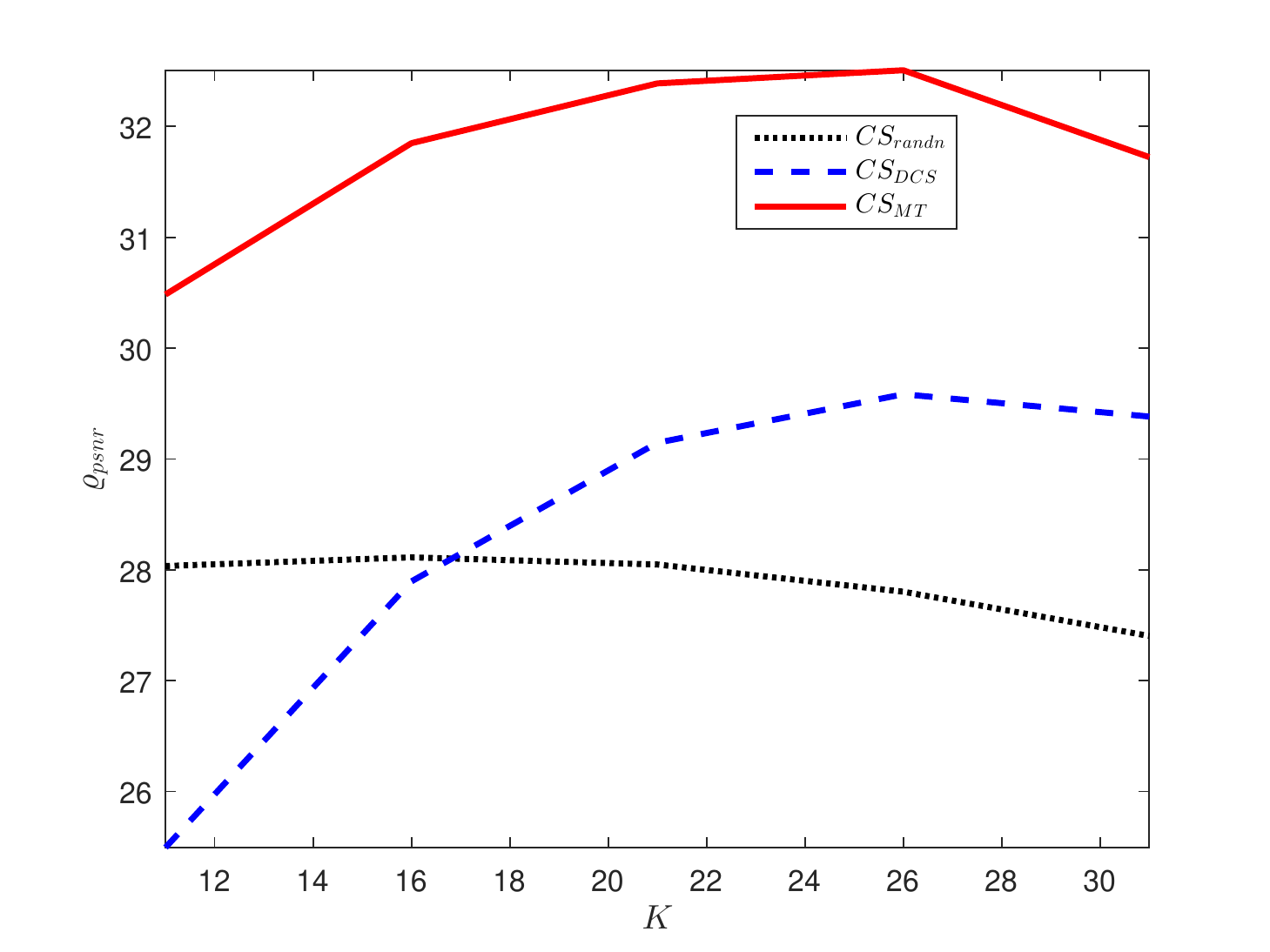}
	\caption{The statistic $\varrho_{psnr}$ versus the varying of $K$ on \emph{Testing Data II} for a fixed $N = 256$, $M=80$ and $L=800$.}\label{Vary_K_High}
\end{figure}
\begin{figure}[htb!]
	\centering
	\includegraphics[width=0.5\textwidth]{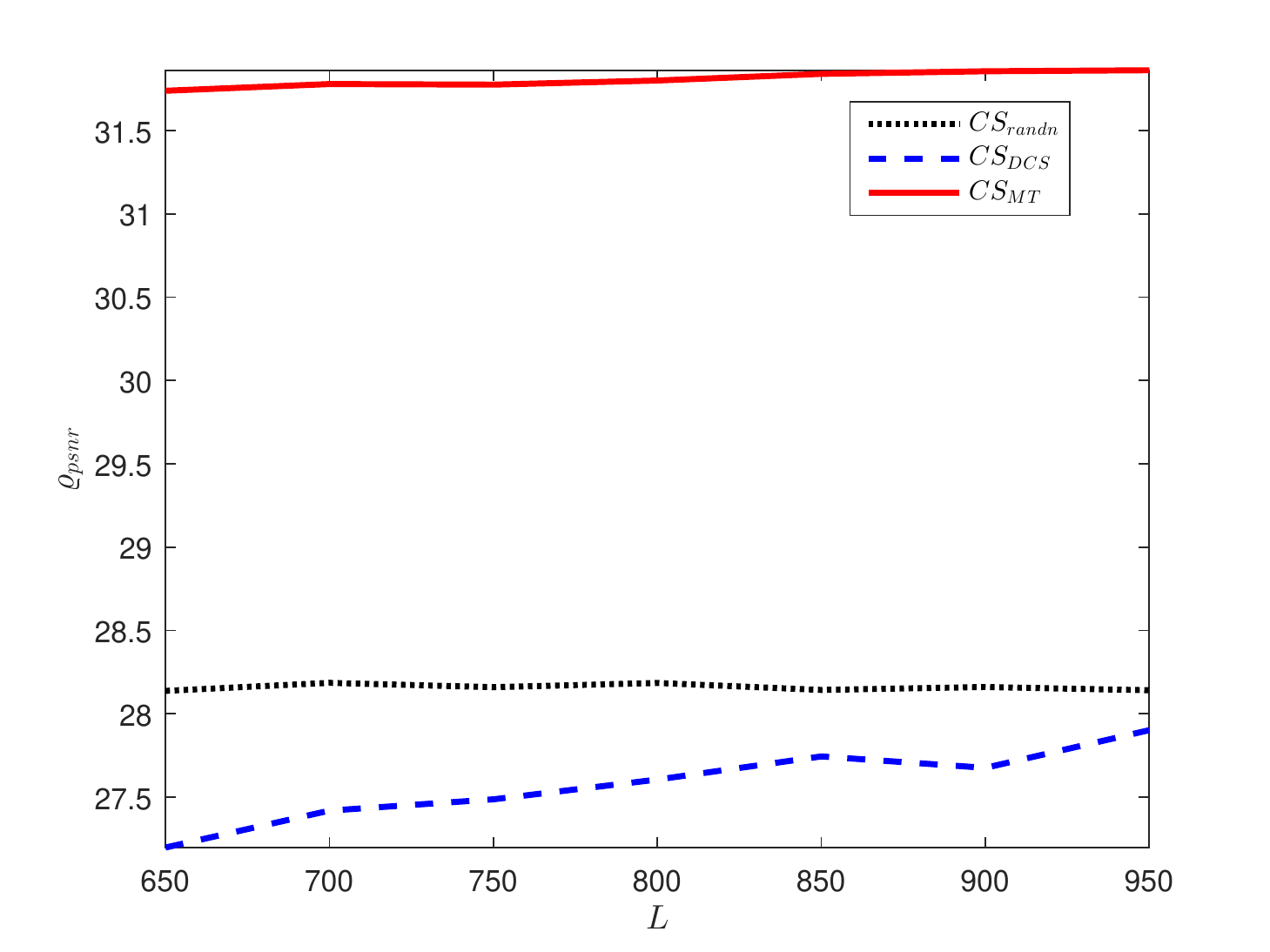}
	\caption{The statistic $\varrho_{psnr}$ versus the varying of $L$ on \emph{Testing Data II} for a fixed $N = 256$, $M=80$ and $K=16$.}\label{Vary_L_High}
\end{figure}

\section{Conclusion}\label{S_5}

{This paper considers the problem of designing a robust projection matrix for the signals that are not exactly sparse. A novel cost function is proposed to decrease the influence of SRE for the measurements and at the same time is independent of training data and the corresponding SRE matrix (the independence of training data saves computations for practical designing). As shown in {\bf Lemma \ref{Lemma:omit_E}}, we state that discarding the SRE matrix in designing procedure is reasonable as it is equivalent to the case when we have infinite number of training samples. We thus utilize $\|\bm \Phi\|_F^2$ as an surrogate to the projected SRE $\|\bm \Phi \bm E\|_F^2$ to design the sensing matrices. The performance of designing projection matrices with dictionaries either learned on large-scale training dataset or of high dimension is experimentally examined. The simulation results on synthetic data and natural images demonstrate the effectiveness and efficiency of the proposed approach. It is of interest to note that the proposed method yields better performance when we increase the dimension of the dictionary, which surprisingly is not true for the other methods.
	
Our proposed framework for designing robust sensing matrices---which shares similar structure to that in \cite{LLLBJH15} \cite{HBLZ16}---simultaneously minimizes the surrogate of sparse representation error (SRE) and the mutual coherence of the CS systems. Thus we need extract effort (like Figure \ref{sythetic_choice_lambda}) to find an optimal $\lambda$ that well balances these two terms. An ongoing research is to come up with a new framework without requiring balancing the tradeoff between minimizing the mutual coherence and decreasing the projected SRE.
}

\section*{Acknowledgment}
This research is supported in part by ERC Grant agreement no. $320649$, and in part by the Intel Collaborative Research Institute for Computational Intelligence (ICRI-CI). The code in this paper to represent the experiments can be downloaded through the link \protect \url{https://github.com/happyhongt/}




\begin{thebibliography}{255}
	
	\bibitem{CRT06} E. J. Cand\`{e}s, J. Romberg, and T. Tao, ``Robust uncertainty principles: Exact signal reconstruction from highly incomplete frequency information,'' {\it IEEE Trans. Inf. Theory,} vol. 52, pp. 489-509, Feb. 2006.
	\bibitem{CT06} E. J. Cand\`{e}s and T. Tao, ``Near optimal signal recovery from random projections: Universal encoding strategies,'' {\it IEEE Trans. Inf. Theory,} vol. 52, pp. 5406-5425, Dec. 2006.
	\bibitem{D06}D. L. Donoho, ``Compressed sensing,'' {\it IEEE Trans. Inf. Theory,} vol. 52, pp. 1289-1306, Apr. 2006.
	\bibitem{CW08}E. J. Cand\`{e}s and M. B. Wakin, ``An introduction to compressive samping,'' {\it IEEE Signal Process. Mag.,} vol. 25, pp. 21-30, Mar. 2008.
	\bibitem{E10}M. Elad, {\it Sparse and Redundant Representations: from theory to applications in signal and image processing,} Springer Science \& Business Media, 2010.
	\bibitem{EK12}Y. C. Eldar and G. Kutyniok, {\it Compressed Sensing: Theory and Application}, Cambridge University Press, May 2012.
	\bibitem{ZE10}M. Zibulevsky and M. Elad, ``$\ell_1$-$\ell_2$ Optimization in Signal and Image Processing,'' {\it IEEE Signal Process. Mag.} vol. 27, pp. 76-88, May 2010.
	
	\bibitem{Barchiesi13}D. Barchiesi, and M. D. Plumbley, ``Learning incoherent dictionaries for sparse approximation using iterative projections and rotations,'' {\it IEEE Trans. Signal Process.},  vol. 61, pp. 2055-2065, Apr. 2013.
	\bibitem{LiICASSP2017}G. Li, Z. Zhu, H. Bai, and A. Yu, ``A new framework for designing incoherent sparsifying dictionaries,'' in {\it IEEE Conf. Acous., Speech, Signal Process.(ICASSP)}, pp. 4416-4420, 2017.
	\bibitem{E07}M. Elad, ``Optimized projections for compressed sensing,'' {\it IEEE Trans. Signal Process.,} vol. 55, pp. 5695-5702, Dec. 2007.
	\bibitem{AFS12}V. Abolghasemi, S. Ferdowsi, and S. Sanei, ``A gradient-based alternating minimzation approach for optimization of the measurement matrix in compressive sensing,'' {\it Signal Process.,} vol. 94, pp. 999-1009, Apr. 2012.
	\bibitem{LH14}W.-S. Lu and T. Hinamoto, ``Design of projection matrix for compressive sensing by nonsmooth optimization,'' {\it  IEEE International Symposium Circuits and Systems (ISCAS),} pp. 1279 - 1282, Jun. 2014.
	\bibitem{LiICSPCC}S. Li, Z. Zhu, G. Li, L. Chang, and Q. Li, ``Projection matrix optimization for block-sparse compressive sensing,'' {\it IEEE Conf. Signal Process., Communicaton and Computation (ICSPCC)}, Aug. 2013.
	
	\bibitem{LZYCB13}G. Li, Z. H. Zhu, D. H. Yang, L. P. Chang, and H. Bai, ``On projection matrix optimization for compressive sensing systems,'' {\it IEEE Trans. Signal Process.,} vol. 61, pp. 2887-2898, Jun. 2013.
	\bibitem{LLLBJH15} G. Li, X. Li, S. Li, H. Bai, Q. Jiang and X. He, ``Designint robust sensing matrix for image compression,'' {\it IEEE Trans. Image Process.,} vol. 24, pp. 5389-5400, Dec. 2015.
	\bibitem{HBLZ16}T. Hong, H. Bai, S. Li, and Z. Zhu, ``An efficient algorithm for designing projection matrix in compressive sensing based on alternating optimization,'' {\it Signal Process.}, vol. 125, pp. 9-20, Aug. 2016.
	
	\bibitem{ZW}Z. Zhu and M. B. Wakin, ``Approximating sampled sinusoids and multiband signals using multiband modulated DPSS dictionaries,'' {\it J. Fourier Analysis Appl.}, pp. 1-40, Aug. 2016.
	
	
	\bibitem{TF11} I. Tosic and P. Frossard, ``Dictionary Learning,'' {\it IEEE Signal Process. Mag.,} vol. 28, pp. 27-38, Mar. 2011.
	
	\bibitem{AEB06}M. Aharon, M. Elad, and A. Bruckstein, ``K-SVD: An algorithm for designing overcomplete dictionaries for sparse representation,'' {\it IEEE Trans. Signal Process.,} vol. 54, pp. 4311-4322, Nov. 2006.
	\bibitem{EAH99}K. Engan, S. O. Aase and J. H. Hakon-housoy, ``Method of optimal direction for frame design,'' {\it Proc. IEEE Int. Conf. Acoust., Speech, Signal Process.(ICASSP),} vol. 5, pp. 2443-2446, Mar. 1999.
	
	\bibitem{B98}L. Botou, ``Online algorithms and stochastic approximations,'' {\it Online Learning and Neural Networks,} Cambridge Univ. Press, 1998.
	\bibitem{MBPS09}J. Mairal, F. Bach, J. Ponce and G. Sapiro, ``Online dictionary learning for sparse coding,'' {\it Proceedings of the 26th annual international conference on machine learning ACM (ICML),} pp. 689-696, 2009.
	\bibitem{MBPS10} J. Mairal, F. Bach, J. Ponce and G. Sapiro, ``Online learning for matrix factorization and sparse coding,'' {\it Journal of Machine Learning,} vol. 11, pp. 19-60, Jan. 2010.
	
	\bibitem{Zhu2017}Z. Zhu, Q. Li, G. Tang, and M. B. Wakin, ``Global Optimality in Low-rank Matrix Optimization,'' {\it arXiv preprint,} arXiv:1702.07945, 2017.
	\bibitem{MS05} M. Schmidt, ``minFunc: unconstrained differentiale multivariate optimization in Matlab.'' \url{https://www.cs.ubc.ca/~schmidtm/Software/minFunc.html}, 2005.
	\bibitem{SOZE16}J. Sulam, B. Ophir, M. Zibulevsky and M. Elad, ``Trainlets: dictionary learning in high dimensions,'' {\it IEEE Trans. Signal Process.,} vol. 64, pp. 3180-3193, Jun. 2016.
	\bibitem{SE16}J. Sulam and M. Elad, ``Large inpainting of face images with trainlets,'' {\it IEEE Signal Processing Letter,} vol. 23, pp. 1839-1843, Dec. 2016.
	
	
	\bibitem{DCS09}J. M. Duarte-Carvajalino and G. Sapiro, ``Learning to sense sparse signals: simultaneous sensing matrix ans sparsifying dictionary optimization,'' {\it IEEE Trans. Image Process.,} vol. 18, pp. 1395-1408, Jul. 2009.
	
	\bibitem{CR2002}G. Casella, and L.B. Roger. {\it Statistical Inference}, Vol. 2. Pacific Grove, CA: Duxbury, 2002.
	
	\bibitem{NW06}J. Nocedal and S. Wright, {\it Numerical Optimization}, Springer, 2006. 
	
	\bibitem{RTF}B. C. Russell, A. Torralba, K. P. Murphy and W. T. Freeman, ``LabelMe: A Database and Web-Based Tool for Image Annotation,'' {\it International Journal of Computation Vision}, vol. 77, pp. 157-173, May 2008.
	\bibitem{RZE08}R. Rubinstein, M. Zibulevsky and M. Elad, ``Efficient implementation of the K-SVD algorithm and the Batch-OMP method,'' {\it Department of Computer Science, Technion, Israel, Tech. Rep.}, 2008.
	\bibitem{AJBG16} A. Mensch, J. Mairal, B. Thirion and G. Varoquaux, ``Dictionary learning for massive matrix factorization,'' {\it Proceedings of the 33th annual international conference on machine learning ACM (ICML),} 2016.
	
\end{thebibliography}
\end{document}